\documentclass[twoside,11pt]{article}

\usepackage{tikz}
\usetikzlibrary{positioning, arrows.meta}
\usepackage{tikz}
\usepackage{blindtext}
\usepackage{jmlr2e}
\usepackage{xcolor}
\usepackage{amsmath}          
\newcommand{\bb}[1]{\boldsymbol{#1}}
\newcommand{\mb}[1]{\mathbb{#1}}

\usepackage{wrapfig}
\usepackage{comment}
\usepackage{booktabs}
\usepackage{algorithm}
\usepackage[noend]{algpseudocode}
\usepackage{multirow}

\usepackage{jmlr2e}
\usepackage{amsmath,amssymb}

\usepackage{lastpage}
\jmlrheading{23}{2025}{1-\pageref{LastPage}}{1/21; Revised 5/22}{9/22}{21-0000}{Li X. Gai K. and Zhang S.}

\ShortHeadings{Deciphering Shortcut Learning}{Li X., Gai K., Zhang S.}
\firstpageno{1}

\begin{document}

\title{Deciphering Shortcut Learning from an Evolutionary Game Theory Perspective}

\author{\name Xiayang Li \email lixiayang@amss.ac.cn \\
       \addr State Key Laboratory of Mathematical Sciences
       \\ Academy of Mathematics and Systems Science\\
       School of Mathematical Sciences\\
       University of Chinese Academy of Science\\
       Beijing, 100049, China\\
       \AND
       \name Kuo Gai \email gaikuo@simis.cn \\
       \addr Shanghai Institute for Mathematics and Interdisciplinary Sciences \\
       Shanghai, 200433, China\\
       \AND
       \name Shihua Zhang \email zsh@amss.ac.cn \\
       \addr State Key Laboratory of Mathematical Sciences
       \\ Academy of Mathematics and Systems Science\\
       School of Mathematical Sciences\\
       University of Chinese Academy of Science\\
       Beijing, 100049, China}

\maketitle

\begin{abstract}
Shortcut learning causes deep learning models to rely on non-essential features within the data. However, its formation in deep neural network training still lacks theoretical understanding. In this paper, we provide a formal definition of core and shortcut features and employ evolutionary game theory to analyze the origins of shortcut bias by  modeling data samples as players and their corresponding neural tangent features as strategies, assuming the existence of core and shortcut subnetworks. We find that gradient descent (GD) and stochastic gradient descent (SGD) lead to two distinct stochastically stable states, each corresponding to a different strategy. The former primarily optimizes the shortcut subnetwork, while the latter primarily optimizes the core subnetwork. We investigate the influence of these strategies on shortcut bias through a continuous stochastic differential equation, and reveal the impact of data noise and optimization noise on the formation of shortcut bias. In brief, our work employs evolutionary game theory to characterize the dynamics of shortcut bias formation and provides a theoretical view on its mitigation.
\end{abstract}

\noindent

\section{Introduction}
Deep learning systems have achieved extraordinary progress in vision, language, and science. Yet their success is often fragile: models frequently adopt shortcuts—spurious correlations in the data that yield good performance on benchmarks but fail under more challenging conditions~\citep{geirhos2020shortcut, arjovsky2019invariant, mccoy2019right, singla2022salient}. Shortcuts can take many forms. For example, in computer vision, convolutional neural networks trained on ImageNet often rely on local textures rather than global shape~\citep{baker2018deep, geirhos2018imagenet, hermann2020origins}, whereas humans recognize objects primarily by shape~\citep{landau1988importance}. This reliance can lead to striking errors, e.g., a hairless cat may be misclassified as an elephant due to skin texture similarity. Scaling up has not solved the problem. Today’s large language models likewise exploit dataset-specific shortcuts, with negative implications for generalization, robustness, and fairness~\citep{yuan2024llms}. Thus, understanding shortcut learning is crucial not only for interpretability and security~\citep{ilyas2019adversarial, song2024shortcut} but also for data attribution~\citep{dogra2024shortcut}, which traces predictions back to influential training instances.

Although much attention has paid to why shortcuts exist, relatively little is known about how they emerge and evolve during training. Early work illustrated that models often rely on shortcuts at initialization~\citep{hermann2024on}, but did not examine their subsequent dynamics. More recently, \citep{jain2024bias} systematically studied the evolution of bias, but only in linear models, overlooking the nonlinear feature-learning process. Existing studies lack precise definitions that separate core features from shortcuts. In practice, shortcuts may include simple cues, such as local textures or superficial statistics, but more subtle patterns, such as a patch intentionally inserted by a backdoor attack~\cite{Manna2024}. Without a clear taxonomy, the mathematical analysis remains ambiguous.

Empirical observations on binary class Colored-MNIST illustrate how shortcut learning evolves in practice. The digit’s color, serving as a shortcut, is learned rapidly at the beginning of training, while the digit category, a core feature, becomes dominant only later (\textbf{Fig.~\ref{Fig_1}a}). When Gaussian noise is injected into the inputs, shortcut bias is amplified, even though the noise itself lacks discriminative power (\textbf{Fig.~\ref{Fig_1}b,c}). This indicates that stochasticity at the data level can preserve shortcut dominance and alter the trajectory of feature learning. These findings connect with broader evidence that large-scale CNNs can memorize random labels and fit random pixels with zero error~\citep{zhang2017understanding}. In this case, taking shortcuts alone is not enough. The network may turn to irrelevant signals in the data rather than the true core features.

Beyond data noise, the stochasticity inherent in optimization also plays a decisive role. Gradient-based optimization can affect generalization, and the step size and batch size have a significant impact on performance. A large body of work has examined the specific effects of small batch size~\citep{mulayoff2020unique, haochen2021shape, damian2021label, wu2022does} and learning rate~\citep{li2019towards, nacson2022implicit, even2023s}. Yet, these studies assume independent and identically distributed (i.i.d.) data and focus primarily on parameter dynamics, leaving open how optimization noise influences the dynamics of feature learning and performance on out-of-distribution (o.o.d.) data. In practice, we find that increasing optimization noise, for example, by decreasing batch size, reduces shortcut bias (\textbf{Fig.~\ref{Fig_1}c}). Taken together, these results suggest that both sources of randomness—data noise and optimization noise—jointly shape the balance between shortcut and core features.

\begin{figure*}[!ht]          
  \centering                 
  \includegraphics[width=0.99\linewidth]{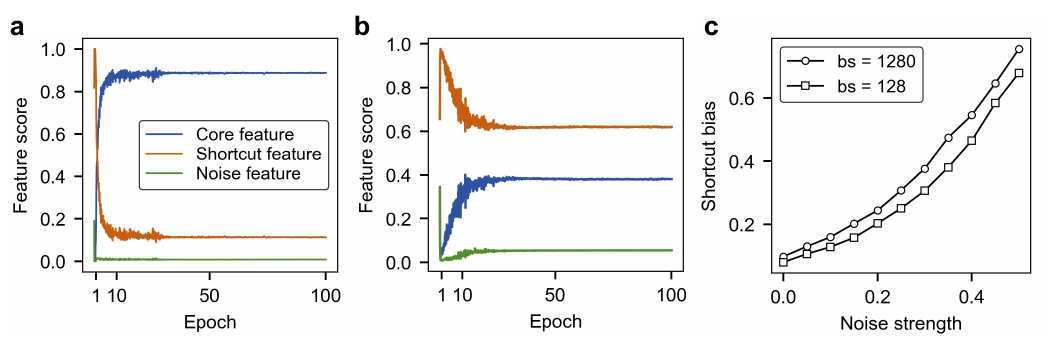} 
  \caption{\textbf{a} and \textbf{b}, Evolution of the three feature types on clean samples \textbf{(a)} and after noise injection (Gaussian noise with a standard deviation of 0.5) \textbf{(b)}. \textbf{c}, Shortcut-bias changes with noise strength under different batch sizes.
}      
  \label{Fig_1}   
\end{figure*}

Neural networks are capable of modeling diverse input features~\citep{frankle2018lottery}, often forming distinct subnetworks during training~\citep{qiu2024complexity, lampinen2024learned}. Which types of features dominate depend on the gradients contributed by individual samples. If the neural tangent feature of a sample, $\nabla_{\bb{\theta}} f(\bb{X};\bb{\theta})$, is interpreted as its strategy for influencing the optimization process, then training can be viewed as a dynamic interaction among competing strategies. This view is consistent with empirical evidence that some difficult samples are repeatedly forgotten and relearned~\citep{toneva2018an}, suggesting that their influence on training self-evolves over iterations. To capture this evolutionary dynamic, we introduce an evolutionary game-theoretic framework that models shortcut and core features as distinct strategies whose prevalence shifts during training.

In this framework, we provide precise definitions of core and shortcut features: shortcuts are features that are easy to learn yet ultimately misleading, while core features are strictly correct. We show that core features are unique (\textbf{Thm.~\ref{the-core}}), whereas shortcuts are ubiquitous (\textbf{Thm.~\ref{the-shortcut}}). Features with higher explained variance are preferentially acquired early in training, granting shortcuts an initial advantage (\textbf{Thm.~\ref{the-subspace}}). Assuming that neural networks contain subnetworks specialized for different features~(\textbf{Assumption~\ref{assu-subnetwork}}), we model each sample as a player, with its neural tangent feature as the strategy. The alignment between a sample’s neural tangent feature and the average batch gradient determines its payoff, reinforcing the corresponding subnetwork. Using this setup, we demonstrate that the stochasticity of optimization determines which strategies become stochastically stable: full-batch gradient descent favors shortcuts, whereas mini-batch stochastic gradient descent favors core features (\textbf{Thm.~\ref{Thm-GD}}, \textbf{Thm.~\ref{the-sgd}}). Extending the analysis with stochastic differential equations, we further connect data and optimization noise to the long-run evolution of feature dominance (\textbf{Thm.~\ref{SDE-Thm}}).

By integrating empirical findings with a formal game-theoretic framework, this work advances a deeper understanding of shortcut learning. It reveals that the balance between shortcuts and core features is not fixed at initialization but shaped dynamically by the interaction of data noise, optimization stochasticity, and gradient competition. These insights unify perspectives from feature learning, stochastic optimization, and robustness, and highlight new theoretical and practical avenues for mitigating shortcut bias in deep learning.

\section{Setting and Methodology}

\subsection{The Formal Definition of Shortcut Feature}
In recent years, shortcut learning has emerged as a central topic in research on robustness and interpretability. However, the research community still lacks a unified, formal definition of shortcut features. Generally, there are three complementary lines of work.

\begin{enumerate}
\item From the perspective of o.o.d. generalization, a shortcut is any feature that yields high accuracy on i.i.d. data but fails under distribution shift~\citep{geirhos2020shortcut}.

\item From the perspective of causal structure, a shortcut is any feature whenever its correlation with the label is non-causal and therefore unreliable once the data-generating mechanism changes~\citep{bellamy2022structural}. 

\item From the perspective of robustness, a shortcut is any feature that is adversarially brittle and non-robust.~\citep{ilyas2019adversarial}.  
\end{enumerate}

We observe that all of the above interpretations of shortcut features inevitably share the following two characteristics. First, it should be easy to learn, as reflected in the ability of neural networks to capture such features from the training dataset and achieve seemingly good performance on standard test sets. Second, it should reflect the incorrectness. This means that if a network learns only such features, it fails to demonstrate a generalizable ability to solve the underlying task. Therefore, our formal definition of a shortcut feature should take both of these requirements into account, which reflects the model’s erroneous reliance on the feature.

Before introducing the formal definition of the shortcut feature, we first clarify the notion of a feature for a classification problem. Note that our definition aligns closely with those proposed in previous work~\citep{lampinen2024learned}.

\begin{definition}[Feature Mapping]
A feature is defined as an abstract, task-relevant property obtained by applying a feature mapping function $V:
\mathbb{R}^d \rightarrow \mathbb{R}^M$ to the input $\bb{X}\in\mathbb{R}^d$, and $M$ is the feature dimension. A feature is considered task-relevant if it assumes exactly as many distinct values (or value vectors) as the number of target classes.
\label{defn-1}
\end{definition}

For example, in the Colored-MNIST dataset, samples can be categorized into two classes based on either color or digit features. To ensure these features exhibit task-relevant properties, both the color and digit features are constrained to take exactly two distinct values. Specifically, the color feature is set to 0 for red and 1 for green. Similarly, the digit feature is assigned a value of 0 if the digit is less than 5, and 1 if it is 5 or greater. Of course, the dataset also contains some noise features. For consistency, we constrain these features to take on exactly two distinct values, which are then randomly assigned to different samples. Restricting each feature to the number of class labels facilitates the formalization of shortcut bias.

Based on the definitions themselves, we can directly characterize core features and noise features. Our fundamental premise is that once the model has fully learned the core features, it can successfully solve the related downstream tasks. In contrast, noise features contribute no benefit to the downstream task. We begin by introducing the Kronecker delta notation $\delta$ for notational convenience:
\begin{equation*}
\delta_{x,y}=\begin{cases}
    1,&x=y\\
    0, & x\neq y.
\end{cases}
\end{equation*}

\begin{definition}[Core and noise feature]
Consider a binary classification problem $\mathcal{P}$ with input space $\mathbb{R}^d$. Let $Y: \mathbb{R}^d \to \{-1,1\}$ denote the ground truth label function, and let $\mathcal{X}$ denote the sample space.  A Feature $V^{\operatorname{core}}$ is defined as core feature if and only if for any input $\bb{X}_i,\bb{X}_j\in \mathcal{X}$, it holds that $\delta_{V^{\operatorname{core}}(\bb{X}_i), V^{\operatorname{core}}(\bb{X}_j)}=\delta_{Y^{\operatorname{core}}(\bb{X}_i), Y^{\operatorname{core}}(\bb{X}_j)}$. A feature $V^{\operatorname{noise}}$ is defined as a noise feature if and only if for any input $(\bb{X}_i,\bb{X}_j)\in\mathcal{X}\times \mathcal{X}$,
$
\operatorname{Cov}\left(\delta_{V^{\operatorname{noise}}(\bb{X}_i),V^{\operatorname{noise}}(\bb{X}_j)},2\delta_{Y(\bb{X}_i),Y(\bb{X}_j)}-1\right)=0.
$
\label{defi_core_feature}
\end{definition}
However, the definition of a shortcut feature requires a comparison between features. Specifically, if feature $V^\alpha$ is a shortcut feature with respect to feature $V^\beta$, then the model exhibits a stronger erroneous reliance on feature $V^\alpha$, which is reflected through the model’s performance on conflicting samples.

\begin{definition}[Shortcut Feature]
Consider a binary classification problem $\mathcal{P}$ with input space $\mathbb{R}^d$ and model $f: \mathbb{R}^d \to \{-1,1\}$. Let $Y: \mathbb{R}^d \to \{-1,1\}$ denote the ground truth label function, and let $\mathcal{X}$ denote the sample space. 
    
Let $V^\alpha$ and $V^\beta$ be two features. Define two conflicting subsets $\Omega_\alpha, \Omega_\beta \subseteq \mathcal{X} \times \mathcal{X}$ as
\begin{equation*}
    \begin{aligned}
    \Omega_\alpha &= \left\{(X_i, X_j) \in \mathcal{X} \times \mathcal{X} \mid V^\alpha(X_i) \neq V^\alpha(X_j),\ V^\beta(X_i) = V^\beta(X_j) \right\}, \\
    \Omega_\beta &= \left\{(X_i, X_j) \in \mathcal{X} \times \mathcal{X} \mid V^\alpha(X_i) = V^\alpha(X_j),\ V^\beta(X_i) \neq V^\beta(X_j)\right\}.
    \end{aligned}
\end{equation*}
And we define $S_k$ as
    \begin{equation*}
    S_k(f) := \mathbb{E}\left[\delta_{f(\bb{X}_i),f(\bb{X}_j)}-\delta_{Y(\bb{X}_i),Y(\bb{X}_j)}\mid (X_i,X_j) \sim \Omega_k\right],
    \end{equation*}
where $k\in\{\alpha,\beta\}$ and $\mathbf{1}\{\cdot\}$ denotes the indicator function. We write $V^\alpha \ll V^\beta$ to indicate that $V^\alpha$ is a shortcut feature relative to $V^\beta$, provided that $S_\alpha(f)<S_\beta(f)$ is satisfied. The shortcut bias is given by $\frac{S_\beta(f) - S_\alpha(f)}{2}$.
\label{defi_shortcut_feature}
\end{definition}

\begin{remark}
    The value of $S_i(f)$ reflects the erroneous reliance on feature $V^i$. Since it is defined over the conflict set $\Omega_i$, where the feature values $V^i$ of the sample pairs differ, a stronger erroneous reliance of the model will lead to a smaller value of $\delta_{f(\bb{X}_i),f(\bb{X}_j)}$, and consequently a smaller $S_i(f)$. It is not difficult to observe that $S_i(f) \in [-1,1]$, and therefore the range of the shortcut bias is $[0,1]$.
\end{remark}

Using the above definitions, we can see the essential property of core features: under the binary relation comparison, they exhibit undominatedness and thus constitute undominated elements.
\begin{theorem}
    A core feature cannot serve as a shortcut feature related to any other feature for any $f$.
    \label{the-core}
\end{theorem}

In the remainder of this paper, we assume that a core feature exists for the given task, denoted by $V^{\text{c}}$. The following theorem establishes that whenever the generalization accuracy is strictly less than one, there must exist a shortcut feature associated with the core feature, denoted by $V^{\text{sc}}$.

\begin{theorem}
    If the core feature $V^{\operatorname{c}}$ exists and $0<\mathbb{E}\left[\delta_{f(\bb{X}_i), Y(\bb{X}_i)}\right]<1$, then there exists a shortcut feature related to $V^{\operatorname{c}}$.
    \label{the-shortcut}
\end{theorem}

The above analysis suggests that existing models almost inevitably face the risk of shortcut learning on a given task. Although the core feature associated with a task may be extremely scarce, there are often many factors that can give rise to shortcut features. In the remainder of this paper, we focus on analyzing a specific class of shortcuts—those that typically emerge during the early stages of model training.

\subsection{Feature Learning via Evolutionary Game Theory}
Neural networks succeed by learning useful representations, but reliable generalization hinges on whether training identifies essential features. Each sample \(\boldsymbol{X}_i\) contributes a tangent feature \(\nabla_{\boldsymbol{\theta}} f(\boldsymbol{X}_i;\boldsymbol{\theta})\) that steers the update. Thus, it is natural to treat
\(\nabla_{\boldsymbol{\theta}} f(\boldsymbol{X}_i;\boldsymbol{\theta}^{(t)})\) as a time–dependent feature. Let $\Delta\bb{\theta}^{(t)}:=\bb{\theta}-\bb{\theta}^{(t)}$, for a differentiable \(f\), if $\Delta$ is sufficiently small, a first-order expansion around \(\boldsymbol{\theta}^{(t)}\) gives
\begin{equation}
\begin{aligned}
f(\boldsymbol{X};\boldsymbol{\theta})
&= f(\boldsymbol{X};\boldsymbol{\theta}^{(t)})
+ \Big\langle \nabla_{\boldsymbol{\theta}} f(\boldsymbol{X};\boldsymbol{\theta}^{(t)}),\, \Delta\boldsymbol{\theta}^{(t)} \Big\rangle
+ o\!\left(\big\|\Delta\boldsymbol{\theta}^{(t)}\big\|_2\right),
\\
&\approx f(\boldsymbol{X};\boldsymbol{\theta}^{(t)})
+ \Big\langle \nabla_{\boldsymbol{\theta}} f(\boldsymbol{X};\boldsymbol{\theta}^{(t)}),\Delta \bb{\theta}^{(t)}\Big\rangle.
\label{eqn:taylor}
\end{aligned}
\end{equation}

Unfortunately, analyzing such dynamics lies beyond the scope of classical theoretical frameworks such as the neural tangent kernel (NTK), because the NTK framework treats $\nabla_\theta f$ as a frozen random feature. In this paper, we model the evolution of these tangent features using evolutionary game theory (EGT). View each sample as a player in a population; strategies correspond to feature directions the network can exploit. Let \(z_t\) be the number of players using strategy \(A\) (versus \ \(B\)) in a symmetric two-strategy game with payoff matrix
\(\boldsymbol{U}=\begin{pmatrix} a & b\\ c & d \end{pmatrix}\).
The aggregate state \(z_t\in\{0,\ldots,N\}\) evolves by
\begin{equation*}
z_{t+1}=b(z_t)+p_t-q_t,
\end{equation*}
where \(b(z_t)\) is the selection step (better-payoff strategies are reinforced), and
\begin{equation*}
p_t\sim\mathrm{Binomial}\!\big(N-b(z_t),\,\varepsilon\big),\qquad
q_t\sim\mathrm{Binomial}\!\big(b(z_t),\,\varepsilon\big)
\end{equation*}
model independent mutations with rate \(\varepsilon\).
Without noise (\(\varepsilon=0\)), the deterministic map \(z_{t+1}=b(z_t)\) converges to \emph{evolutionary stable states} (ESS). With \(\varepsilon>0\), \(\{z_t\}\) is an ergodic Markov chain on \(N{+}1\) states with stationary distribution \(\mu_{\varepsilon}\) concentrating, as \(\varepsilon\to0\), on the \emph{stochastically stable states} (SSS).

To connect with feature learning, decompose tangent features and update directions into core (c) and shortcut (sc) components:
\(\{\nabla_{\boldsymbol{\theta}} f^{(c,t)},\,\nabla_{\boldsymbol{\theta}} f^{(sc,t)}\}\) and
\(\{\Delta\boldsymbol{\theta}^{(\mathrm{c},t)},\,\Delta\boldsymbol{\theta}^{(\mathrm{sc},t)}\}\).
Using \eqref{eqn:taylor}, define the payoff generated by their alignment, i.e., a higher payoff corresponding to a greater reduction in loss of a sample,
\[
P^{a,b,t}:=\operatorname{payoff}\!\left(
\big\langle \nabla_{\boldsymbol{\theta}} f^{a}(\boldsymbol{X}_i;\boldsymbol{\theta}^{(t)}),\,
\Delta\boldsymbol{\theta}^{(b,t)}\big\rangle,\,\cdot\right),\quad a,b\in\{c,sc\},
\]
and the induced payoff matrix
\begin{equation}
\boldsymbol{U}^{(t)}=\begin{pmatrix}
P^{c,c,t} & P^{c,sc,t}\\[2pt]
P^{sc,c,t} & P^{sc,sc,t}
\end{pmatrix}.
\label{pay-off-matrix}
\end{equation}
We assume the aggregate evolution obeys
\begin{equation}
\begin{aligned}
z_{t+1}&=b(z_t)+p_t-q_t,\\
b(z_t)-z_t&=\operatorname{sign}\!\big(\pi(z_{t+1})-\pi(z_t)\big),\\
p_t&\sim \operatorname{Bin}\!\big(N-b(z_t),\varepsilon\big),\\
q_t&\sim \operatorname{Bin}\!\big(b(z_t),\varepsilon\big),
\end{aligned}
\label{darwanian}
\end{equation}
where \(\pi(\cdot)\) is the population payoff induced by \eqref{pay-off-matrix}.
This EGT formulation compresses the narrative: gradient updates enact replication; payoffs are loss reductions from feature–update alignment; long-run behavior is characterized by SSS as \(\varepsilon\to0\).
Next, we analyze the properties of the resulting sample strategies.

\subsection{Model Assumption and Analysis}
We consider fully connected ReLU neural networks of depth $L\in\mathbb{N}$ and widths $(n_l)_{0\leqslant l \leqslant L}$, where $n_0=d$ is the input dimension and $n_L=1$ is the output dimension. Forward expressions are defined as follows:
\begin{equation}
\begin{aligned}
    \boldsymbol{\boldsymbol{f}}^l (\boldsymbol{x})&:=\sqrt{\frac{2}{m_{\ell-1}}}\phi(\boldsymbol{h}^l(\boldsymbol{x})),\quad \boldsymbol{f}^0(\boldsymbol{x}):=\boldsymbol{x}\in\mathbb{R}^{d}\\
    \boldsymbol{h}^l(\boldsymbol{x})&:=\boldsymbol{W}^l\boldsymbol{f}^{l-1}(\boldsymbol{x})+\boldsymbol{b}^{l},\quad 1\leqslant l\leqslant L-1\\
f(\boldsymbol{x})&:=\boldsymbol{W}^{L}\boldsymbol{f}^{L-1}(\boldsymbol{x})+\boldsymbol{b}^L\in\mathbb{R}
\end{aligned}
\label{eqn1}
\end{equation}
where $W_{ij}\sim \mathcal{N}(0,1).$

Consider a binary classification problem, where the input data \(\bb{X}_i \in \mb{R}^d\) and the corresponding labels \(y_i \in \{-1, +1\}\). Let the neural network model be denoted by \(f\), with scalar outputs \(f(\bb{X}_i)\in\mathbb{R}\). Assume the model has \(p\) parameters and the training set contains \(N\) samples. Denote the loss function by \(\ell\), which we take to be the mean squared error (MSE): $\ell(\bb{\theta}) \;=\; \frac{1}{2N}\sum_{i=1}^N \bigl(y_i - f(\bb{X}_i;\bb{\theta})\bigr)^2.$

In our framework, each sample is conceptualized as an entity with characteristics. For simplicity, in our theoretical analysis, we assume that each sample carries exactly two features: a core feature and a shortcut feature. Each sample can employ different strategies, utilizing distinct features to guide the training of the network's parameters. Tangent feature mappings naturally model this guiding effect, as the tangent features are closely related to the gradient vectors: $\bb{X} \rightarrow\nabla_{\bb{\theta}} f(\bb{X};\bb{\theta}^{(t)})$.

At the step $t$, the neural tangent feature matrix $\bb{\operatorname{NT}}$ is defined as
\begin{equation*}
    \bb{\operatorname{NT}}^{(t)}:=[\nabla_{\bb{\theta}} f(\bb{X}_1;\bb{\theta}^{(t)});\cdots;\nabla_{\bb{\theta}} f(\bb{X}_N;\bb{\theta}^{(t)})]\in \mathbb{R}^{N\times p}.
\end{equation*}
If $\bb{\operatorname{NT}}$ has rank $r\leq \min\{N,p\}$ with singular value decomposition $\bb{\operatorname{NT}}^{(t)} = \bb{U}^{(t)} \bb{\Sigma}^{(t)} \bb{V}^{(t)},$ then 
\begin{equation*}
\begin{aligned}
        \begin{pmatrix}
        f(\bb{X}_1;\bb{\theta})\\
        f(\bb{X}_2;\bb{\theta})\\
        \vdots\\
        f(\bb{X}_N;\bb{\theta})
        \end{pmatrix}&= \begin{pmatrix}
        f(\bb{X}_1;\bb{\theta}^{(t)})\\
        f(\bb{X}_2;\bb{\theta}^{(t)})\\
        \vdots\\
        f(\bb{X}_N;\bb{\theta}^{(t)})
        \end{pmatrix}+ \bb{U}^{(t)}\bb{\Sigma}^{(t)}(\bb{V}^{(t)})^{T}(\bb{\theta}-\bb{\theta}^{(t)}) + o(\Vert \bb{\theta}-\bb{\theta}^{(t)}\Vert_2).
\end{aligned}
\end{equation*}
Thus, at time $t$, the direction in which $\theta$ progresses is mainly associated with the larger right singular vectors of $\bb{\operatorname{NT}}$. Therefore, our primary focus is on the behavior of the leading singular vector. A dataset may contain more than one shortcut feature and more than one core feature; nevertheless, in the analysis that follows, we restrict our attention to the simplified case involving a single core feature and a single shortcut feature. Thus, we consider two sub-networks to model the shortcut feature and the core feature, respectively. We assume that two subnetworks have already emerged in the early phase of training; samples that adopt different strategies will eventually train different subnetworks. These subnetworks are decoupled from each other and can identify the core features and the shortcut features, respectively. For a fully connected neural network, if we disregard the bias terms, each element in the weight vector can be mapped one-to-one to an edge. We can make structural assumptions about subnetworks by partitioning the neurons and edges into sets (\textbf{Fig.~\ref{fig:subnetwork}}).

\begin{assumption}
We make the following assumptions about the subnetworks. We first divide all neurons in the network into core neurons and shortcut neurons, which are mutually disjoint. Based on this, the edges can be classified into the following subsets: \\
1. \textbf{(Decoupling)} 
\begin{itemize}
    \item $\mathcal{E}_1$: Edges connecting core neurons to core neurons, core neurons to the input and output. Within $\mathcal{E}_1$, $\mathcal{E}_1^f$ denotes the edges that model the core feature, and $\mathcal{E}_1^n$ denotes the edges that model noise.
    \item $\mathcal{E}_2$: Edges connecting shortcut neurons to shortcut neurons, shortcut neurons to the input and output. Within $\mathcal{E}_2$, $\mathcal{E}_2^f$ denotes the edges that model the shortcut feature, and $\mathcal{E}_2^n$ denotes the edges that model noise.
    \item $\mathcal{E}_3$: Edges connecting core neurons to shortcut neurons.

\label{assu-subnetwork}
\end{itemize} 2. \textbf{(Feature recognizability)} $V^{k}(X_i)=V^{k}(X_j)$ if and only if $f_{\bb{\theta}|_{\mathcal{E}_k^f}}(X_i)=f_{\bb{\theta}|_{\mathcal{E}_k^f}}(X_j)$, $k\in\{1,2\}$, where $V^1$ denotes a core feature and $V^2$ denotes a shortcut feature. We use $\bb{\theta}|_{\mathcal{E}_k}$ to denote the restriction of the vector $\bb{\theta}$ to the index set $\mathcal{E}_k$, resulting in a new vector of the same dimension where all elements outside of $\mathcal{E}_k$ are set zeros.
\end{assumption}

\begin{figure}[h]
    \centering
    \includegraphics[width=0.6\linewidth]{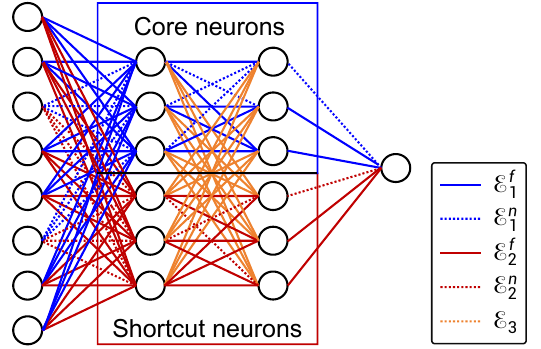}
    \caption{Schematic diagram of the sub-network hypothesis. The hidden layer neurons are divided into two categories: core neurons and shortcut neurons. Based on the classification of neurons, the connecting edges (weights) are divided into three categories. In the two major categories, there are two sub-categories $\mathcal{E}_1^f$ and $\mathcal{E}_1^n$ ($\mathcal{E}_2^f$ and $\mathcal{E}_2^n$) that model core (shortcut) features and noise for $\mathcal{E}_1$ ($\mathcal{E}_2$), respectively.}
    \label{fig:subnetwork}
\end{figure}
\paragraph{Rationale for the assumption} Neural networks exhibit strong feature-extraction abilities. Accordingly, we posit that distinct subnetworks model different features. Note that not all edges in $\mathcal{E}_i,i\in\{1,2\}$ are used for modeling features. For example, the color feature can serve as a shortcut in Color-MNIST; however, a purely color-based subnetwork is unable to discriminate among classes in the training set, because it cannot assign two training samples that share the same color to different categories. Therefore, in the set $\mathcal{E}_i$, besides the edges that model features, there must be some edges that model noise to achieve discriminability. Therefore, we denote the edges that model features as $\mathcal{E}_i^f$, and the edges that model noise as $\mathcal{E}_i^n$.

\section{Theoretical Results}
In this section, we introduce a simplified model to systematically analyze how samples, viewed as game-theoretic players, steer the evolution of the network’s structure.

\subsection{Initialization — Subnetwork Formation}

\begin{theorem}[\cite{bietti2019inductive}]
    The network is defined by \textbf{Eq.}~(\ref{eqn1}). For an input $\bb{x}\in\mathbb{R}^d$ and learnable parameter $\theta\in\mathbb{R}^m$, denote the network by $f(\bb{X},\theta)$, then the corresponding neural tangent kernel is given by $$\bb{k}^{(L)}(\bb{X}_i,\bb{X}_j)=\mb{E}_{\bb{\theta}\sim\mathcal{P}}\left\langle \nabla_{\bb{\theta}} f(\bb{X}_i,\bb{\theta}), \nabla_{\bb{\theta}} f(\bb{X}_j,\bb{\theta})\right\rangle$$
$\bb{k}^{(L)}$ is homogenous of degree 1 and zonal and defined by $\bb{k}^{(L)}(u)=\frac{1}{L+1}\tilde{\bb{k}}^{(L)}(u),$
    with the recursive formula
    \begin{equation*}
    \begin{aligned}
    \tilde{\bb{k}}^{(\ell)}(u)&=\tilde{\bb{k}}^{(\ell-1)}(u)\kappa_0(\Sigma^{(\ell-1)}(u))+\Sigma^{(\ell)}(u)\\
    \Sigma^{(\ell)}(u)&=\kappa_1(\Sigma^{(\ell-1)}(u)),~\ell\in [L].
    \end{aligned}
    \end{equation*}
The function $\kappa_1,\kappa_0$ are arc-cosine kernels, defined as
    \begin{equation*}
    \begin{aligned}
        &\kappa_0(u)=\frac{1}{\pi}(\pi-\arccos (u)),\\
        &\kappa_1(u)=\frac{1}{\pi}(u\cdot(\pi-\arccos(u))+\sqrt{1-u^2})\\
        &\tilde{\bb{k}}^{(0)}(u)=\Sigma^{(0)}(u)=u.
    \end{aligned}
    \end{equation*}
    \label{the-subspace}
\end{theorem}
If we consider a two-layer network, the kernel function is given by $$\bb{k}^{(2)}(\bb{X}_i,\bb{X}_j)=\Vert \bb{X}_i\Vert_2\Vert \bb{X}_j\Vert_2 h(\left\langle \bb{X}_i,\bb{X}_j\right\rangle),~h(u)=\frac{1}{\pi}(u(\pi-\arccos(u))+\frac{1}{2}\sqrt{1-u^2}).$$ Using a quadratic approximation of $h$, we get  $\hat{h}(u)=\frac{3}{4\pi}u^2+\frac{1}{2}u+\frac{1}{2\pi}$. We will show that the principal eigenvector of $\bb{K}_{\bb{\theta}}$ ($K_{i,j}:=\bb{k}(\bb{X}_i,\bb{X}_j)$) is tightly linked to the principal eigenvector of $\bb{X}\bb{X}^{\top}$.

\begin{theorem}
\label{thm:cluster_alignment}
Let $\{\bb{X}_i\}_{i=1}^N \subset \mathbb{R}^d$ satisfy 
$\Vert \bb{X}_i \Vert_2 = 1$, Denote the Gram matrix by $\bb{K}_X = \bb{XX}^\top\in\mathbb{R}^{N\times N}$, whose eigen-values obey $\lambda_1\geqslant \lambda_2\geqslant \cdots \geqslant \lambda_r>\lambda_{r+1}\geqslant \lambda_{r+2}\geqslant \cdots \geqslant \lambda_{N}, ~\delta:=\lambda_r-\lambda_{r+1}>0.$
Define the kernel \begin{equation*}\hat h(u) \;=\; \frac{3}{4\pi}\,u^{2} \;+\;\frac12\,u \;+\;\frac{1}{2\pi},
(\boldsymbol{K}_{\!\theta})_{ij}
=
\hat h\!\bigl(\left\langle\boldsymbol{X}_i,\boldsymbol{X}_j\right\rangle\bigr).\end{equation*}

\smallskip
\noindent
Denote by $U \in\mathbb{R}^{N\times r},
\quad
V \in\mathbb{R}^{N\times r},$
the orthonormal eigen‐bases whose columns span, respectively,
the leading $r$-dimensional spectral subspace
  $\mathcal{U}=\operatorname{span}\{\boldsymbol u_{1},\dots,\boldsymbol u_{r}\}$ of
  $\boldsymbol K_{\!X}$, and  
the corresponding subspace
  $\mathcal{V}=\operatorname{span}\{\boldsymbol v_{1},\dots,\boldsymbol v_{r}\}$
  of $\boldsymbol K_{\!\theta}$.

If $\rho:=\max_{i\neq j}\bigl|\left\langle\boldsymbol{x}_i,\boldsymbol{x}_j\right\rangle\bigr|\;\in[0,1]$, then the operator–angle distance between the two subspaces satisfies
\begin{equation*}
\;
\bigl\|\sin\Theta(\mathcal U,\mathcal V)\bigr\|_{2}
\;=\;
\bigl\|U^{\!\top}V_{\perp}\bigr\|_{2}
\;\le\;
\frac{3}{2\pi}\,
\frac{1+(N-1)\rho^{2}}{\delta}\; .
\end{equation*}
Here $V_{\perp}$ is any orthonormal basis of $\mathcal{V}^{\perp}$,
and $\sin\Theta$ is the canonical principal‐angle operator.
\end{theorem}

This suggests that when shortcut features in the data dictate the orientation of the leading principal component, or, more broadly, the low-dimensional subspace spanned by the top eigenvalues, neural networks tend to latch onto them very quickly. The next theorem shows that when noise is injected into the data and the core features are not aligned with the principal directions of the data-kernel matrix $\bb{K}_{\bb{X}}$, those features may be distorted, or even overwhelmed, by the noise.

\begin{theorem}
    Assume that $\bb{K}_X =\bb{K}_S +\bb{K}_W\in\mathbb{R}^{N\times N}$, where $K_s = \sum\limits_{k=1}^r\beta_k \bb{v}_k\bb{v}_k^\top$ with fixed rank $r\in\mathbb{N}$, pairwise orthonormal vectors $v_1,\cdots,v_r\in\mathbb{R}^n$ and real amplitudes $\beta_1\geq\beta_2\geq\cdots\geq\beta_r>0$ independent of $N$. The noise part $\bb{K}_W=\frac{\bb{G}_N+\bb{G}_N^\top}{\sqrt{2N}}, (G_N)_{i,j}\sim_{\operatorname{i.i.d.}}\mathcal{N}(0,\sigma^2).$ Denote the ordered eigenvalues of $\bb{K}_X$ by $\lambda_1\geq\cdots\ge\lambda_N$, and let $\hat{\bb{v}}_k$ be a unit eigenvector associated with $\lambda_k(\bb{K}_X)$. If a spike amplitude satisfies $\beta_k\leq \sigma$, then for any fixed $\varepsilon>0$,
    \begin{equation}
    \mathbb{P}[\lambda_k(\bb{K}_{\bb{X}})\leq 2\sigma]\rightarrow 1,
    \end{equation}
which means that the $k$-th signal eigenvalue stays inside the wigner bulk $[-2\sigma,2\sigma]$. And $\sqrt{N}\left|\left\langle \bb{v}_k,\hat{\bb{v}}_k \right\rangle\right|\xrightarrow[N\to\infty]{d}\mathcal{N}(0,\sigma^2)$, $\left\langle \bb{v}_k,\hat{\bb{v}}_k\right\rangle=O_{\mathbb{P}}(\frac{1}{\sqrt{N}})$. Conversely, if $\beta_k>\sigma$, then with probability $1-o(1)$, $\lambda_k(\bb{K}_X)\xrightarrow[N\rightarrow\infty]{}\beta_k+\frac{\sigma^2}{\beta_k}, \left|\left\langle \bb{v}_k,\hat{\bb{v}}_k\right\rangle\right|^2\xrightarrow[N\rightarrow \infty ]{}1-\frac{\sigma^2}{\beta_k^2}.$
\label{the-noise}
\end{theorem}

\subsection{Training — Subnetwork Reinforcement}
The preceding analysis shows that early‐stage neural networks rapidly capture the shortcut features; however, this does \emph{not} imply that they are incapable of modeling the core features. But the subnetworks responsible for modeling the core features fail to dominate the network. In this subsection, we adopt an evolutionary-dynamics perspective to model the strategy of the samples, namely, how the samples guide the update of network parameters.

Different samples within each training batch guide the model updates through their tangent features, $\bb{X}\mapsto\nabla_{\bb{\theta}} f(\bb{X})$. Let $\boldsymbol{\Omega_1}$ be the set of samples that guide the model via core features, and let $\boldsymbol{\Omega_2}$ be the set that guides it via shortcut features. Denote by $z_t:= |\bb{\Omega_1}|$, the number of individuals following the core-feature strategy at step $t$. After the $t$-th batch update, every sample receives a payoff equal to the drop in its own loss.

\begin{definition}
 The payoff of a sample $X$ at time $t$ is defined as the decrease in its loss if it is included in the mini-batch at time $t$, and $0$ otherwise. Formally,
\[
P(\bb{X})=\mathbf{1}\{\bb{X}\in \mathcal{B}_t\}\left(\ell^{t-1}(\bb{X})-\ell^{(t)}(\bb{X})\right).
\]
\end{definition}
The payoff can be regarded as the feedback from optimization to the sample’s strategy. A strategy with a higher payoff indicates that it better matches the task label at the current step. $P(\bb{X})$ will influence the training dynamics of the network. Next, we focus on two key questions:
\begin{enumerate}
\item[(1)] How do samples leverage their tangent features to guide the network’s updates, especially in guiding the learning of the two feature-specific subnetworks?
\item[(2)] How does the evolution of $z_t$ feed back into overall network behavior; in particular, does it foster or suppress the emergence of shortcut bias?
\end{enumerate}

\subsubsection{Tangent Projection}
First, we take a Taylor expansion of the network at time step $t$ ($t\ge 1$):
$$f(\bb{X};\bb{\theta}^{(t+1)}) = f(\bb{X};\bb{\theta}^{(t)}) + \left\langle \nabla_{\bb{\theta}} f(\bb{X};\bb{\theta}^{(t)}), \Delta\bb{\theta}^{(t)} \right\rangle + \frac{1}{2} \Delta{\bb{\theta}^{(t)}}^T H(\bb{\theta}^*) \Delta{\bb{\theta}^{(t)}}$$
where $\Delta\bb{\theta}^{(t)} := \bb{\theta}^{(t+1)}-\bb{\theta}^{(t)},\bb{\theta}^*\in [\bb{\theta}^{(t)}, \bb{\theta}^{(t+1)}]$. If the function is locally $L$-smooth with respect to $\bb{\theta}$, i.e., $\Vert\nabla f(X;\bb{\theta}^{(t+1)})-\nabla f(X;\bb{\theta}^{(t)})\Vert_2\leqslant L\Vert \bb{\theta}^{(t+1)}-\bb{\theta}^{(t)}\Vert_2$, and $\Delta \bb{\theta}$ is sufficiently small, then we can use a first-order Taylor expansion:$
    f(\bb{X};\bb{\theta}^{(t+1)}) = f(\bb{X};\bb{\theta}^{(t)}) + \left\langle \nabla_{\bb{\theta}} f(\bb{X};\bb{\theta}^{(t)}), \Delta\bb{\theta}^{(t)} \right\rangle + o(\Vert \Delta \bb{\theta}^{(t)}\Vert_2)$.
According to the sub-network hypothesis (\textbf{Assumption~\ref{assu-subnetwork}}), we restrict the change in parameters to the set of edges of the two sub-networks, resulting in:
\begin{equation*}
\begin{aligned}
\Delta\bb{\theta}^{(t)}&=\Delta\bb{\theta}^{(t)}|_{\mathcal{E}_1} +\Delta \bb{\theta}^{(t)}|_{\mathcal{E}_2} + \Delta \bb{\theta}^{(t)}|_{\mathcal{E}_3}\\
& := \Vert \Delta\bb{\theta}^{(t)}|_{\mathcal{E}_1}\Vert _2\bb{V}_1^{(t)} +\Vert \Delta\bb{\theta}^{(t)}|_{\mathcal{E}_2}\Vert _2\bb{V}_2^{(t)}+ \Vert \Delta\bb{\theta}^{(t)}|_{\mathcal{E}_3}\Vert _2\bb{V}_3^{(t)}.
\end{aligned}
\end{equation*}
We treat the Euclidean norm of the projection of $\Delta\bb{\theta}^{(t)}$ onto $\mathcal{E}_i$ as the network’s training intensity on subnetwork $\mathcal{E}_i$ at step $t$, $i\in\{1,2\}$. $\bb{V}_i$ is the unit vector in the direction of $\Delta \bb{\theta}$. We denote the training intensities for the two sub-networks at time $t$ as $\Delta w_1^{(t)}$ and $\Delta w_2^{(t)}$, and the overall training intensities up to $t$ as 
$w_1^{(t)}:=\sum\limits_{s=1}^t \Delta w_1^{(s)}$ and $w_2^{(t)}:=\sum\limits_{s=1}^t \Delta w_2^{(s)}$.

At the same time, since we assume $\mathcal{E}_1,\mathcal{E}_2$ and $\mathcal{E}_3$ are pairwise disjoint, this implies that $\bb{V}_1^{(t)},\bb{V}_2^{(t)},\bb{V}_3^{(t)}$ are mutually orthogonal unit vectors. Then we decompose the neural-tangent feature as $\mathrm{Proj}_{\Delta\bb{\theta}^{(t)}}\nabla_{\bb{\theta}}f(\bb{X};\bb{\theta}^{(t)})
  \;=\sum\limits_{i=1}^3\;
  \Bigl\lVert \mathrm{Proj}_{\bb{V}_i^{(t)}}\nabla_{\bb{\theta}}f(\bb{X};\bb{\theta}^{(t)})
  \Bigr\rVert_2\,
  \bb{V}_i^{(t)}$.
Each sample's neural tangent feature corresponds to a strategic direction. Next, we mainly model two leading directions, which are governed by the core features and shortcut features, respectively.

\begin{figure}[h]
    \centering
    \includegraphics[width=0.95\linewidth]{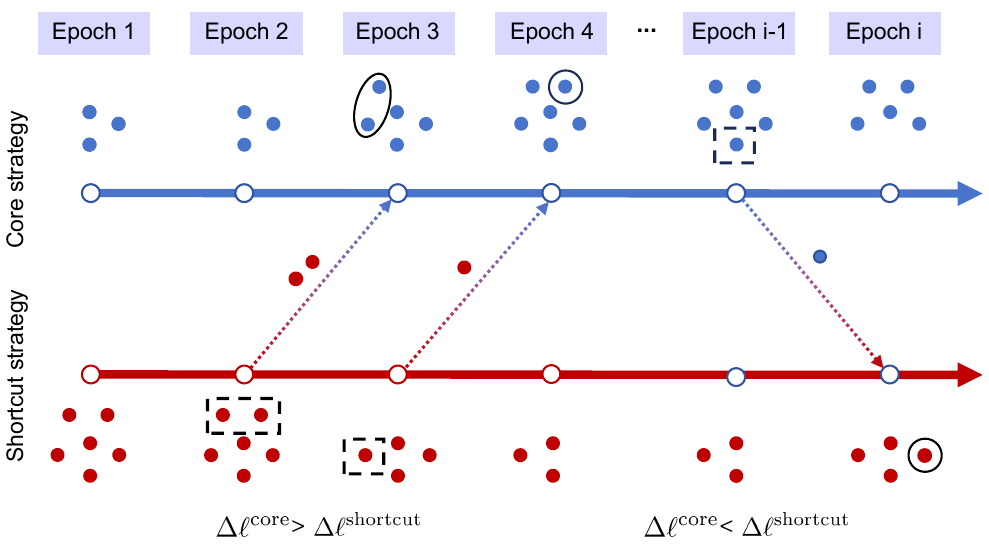}
    \caption{A schematic diagram of strategy transfer during the training process. Under the sub-network hypothesis, consider that the samples have two strategies: a core strategy and a shortcut strategy. The corresponding evolutionary paths are marked in blue and red, respectively. When the payoff of one strategy is higher, the population size adopting that strategy will increase at the next epoch.}
    \label{fig:3}
\end{figure}

\begin{assumption}[Strategy set]
\begin{itemize}
  \item[(1)] At any epoch \(t\), each sample can choose between two pure strategies, which is represented by tangent features, $\nabla_{\bb{\theta}} f^{(c)}$ and $\nabla_{\bb{\theta}} f^{(sc)}$, where $\nabla_{\bb{\theta}} f^{(c)}$ primarily optimizes the core sub-network. Whereas strategy $\nabla_{\bb{\theta}} f^{(sc)}$ primarily optimizes the shortcut sub-network. 
  \item[(2)] The two strategy vectors are orthogonal and have equal norms when restricted to the edge set $\mathcal{E}_1\bigcup \mathcal{E}_2$. 
  \item[(3)] At each time step $t\geq 1$, the norm of the projected strategy vector $\text{Proj}_{\Delta\bb{\theta}^{(t)}|_{\mathcal{E}_1\cup \mathcal{E}_2}}\nabla_{\bb{\theta}} f$ and its angle with the $\Delta\bb{\theta}^{(t)}\big |_{\mathcal{E}_1\cup\mathcal{E}_2}$ both remain constant.
\end{itemize}
\label{assump_2}
\end{assumption}

After the first epoch, we assume the following distribution for the
neural-tangent directions induced by different features. It is the result of orthogonally decomposing $\nabla_{\bb{\theta}} f$ with respect to two bases $\bb{V}_1$ and $\bb{V}_2$:
\begin{equation}
\begin{aligned}
g^{(c,+)} &= -\tfrac{\gamma}{\sqrt{1+\gamma^{2}}}\,\bb{V}_2^{(1)}
             + \tfrac{1}{\sqrt{1+\gamma^{2}}}\,\bb{V}_1^{(1)},
   &&\text{w.p. }\tfrac{\alpha^{(1)}}{2},\; y=1,\\[0.3em]
g^{(sc,+)} &= \tfrac{1}{\sqrt{1+\gamma^{2}}}\,\bb{V}_2^{(1)}
             + \tfrac{\gamma}{\sqrt{1+\gamma^{2}}}\,\bb{V}_1^{(1)},
   &&\text{w.p. }\tfrac{1-\alpha^{(1)}}{2},\; y=1,\\[0.3em]
g^{(c,-)} &= \tfrac{\gamma}{\sqrt{1+\gamma^{2}}}\,\bb{V}_2^{(1)}
             - \tfrac{1}{\sqrt{1+\gamma^{2}}}\,\bb{V}_1^{(1)},
   &&\text{w.p. }\tfrac{\alpha^{(1)}}{2},\; y=-1,\\[0.3em]
g^{(sc,-)} &= -\tfrac{1}{\sqrt{1+\gamma^{2}}}\,\bb{V}_2^{(1)}
              - \tfrac{\gamma}{\sqrt{1+\gamma^{2}}}\,\bb{V}_1^{(1)},
   &&\text{w.p. }\tfrac{1-\alpha^{(1)}}{2},\; y=-1.
\end{aligned}
\end{equation}
For notational convenience, define \begin{equation}g^{(c,+)} := 
  \frac{\text{Proj}_{\Delta \bb{\theta}^{(1)}|_{\mathcal{E}_1\cup \mathcal{E}_2}} 
        \nabla_{\bb{\theta}} f^{(c,+)}}
       {\bigl\Vert \text{Proj}_{\Delta \bb{\theta}^{(1)}|_{\mathcal{E}_1\cup \mathcal{E}_2}}
        \nabla_{\bb{\theta}} f^{(c,+)} \bigr\Vert_2},
\quad \end{equation}
and analogously for $g^{(sc,+)}$,$g^{(c,-)}$ and $g^{(sc,-)}$.

For $\nabla_{\bb{\theta}} f^{(\cdot,\cdot)}$, the second subscript $+~$or$~-$ indicates the class to which the sample belongs. The parameter $\gamma$ represents the ratio of the optimization strengths applied to two different sub-networks, corresponding to the strategy adopted by the sample. We denote $M:=\Vert\text{Proj}_{\Delta \bb{\theta}^{(1)}|_{\mathcal{E}_1\cup \mathcal{E}_2}} \nabla_{\bb{\theta}} f\Vert_2$ in the following sections.

\subsubsection{Evolutionary Dynamics}
To capture how the training dynamics select among samples, we assume that the mismatch between a sample’s strategy at step~$t$ and the subnetwork currently driving the update influences the sample’s next decision. Here, we first analyze how to decompose the model's output into the accumulated contributions. We consider two cases.

\paragraph{No evolution of the strategy occurs along the iterative path.} That is, for the current sample, a single strategy, e.g., a core strategy, has been consistently used to guide the network's optimization from the first epoch. In this scenario, we can assume that along the optimization path, i.e., $\frac{\partial^2 f(\bb{X},\bb{\theta})}{\partial \bb{\theta}^2}$, remains bounded. Therefore,
\begin{equation*}
\begin{aligned}
    f^{(t)}(\bb{X}_i) &= f^{(1)}(\bb{X}_i)+\sum\limits_{s=1}^{t-1}\left\langle \nabla_{\bb{\theta}} f(\bb{X}_i;\bb{\theta}^{(s)})^{(c,+)},\Delta \boldsymbol{\theta}^{(s)}\right\rangle+\sum\limits_{s=1}^{t-1} o(\Vert \Delta\bb{\theta}^{(s)}\Vert_2).
\end{aligned}
\end{equation*}

\paragraph{Evolution of the strategy occurs along the iterative path.} That is, at time $k<t$, the strategy evolved the first time; for example, it switched from strategy $\frac{\partial f}{\partial \bb{\theta}}^{(\text{c},+)}$ to strategy $\frac{\partial f}{\partial \bb{\theta}}^{(\text{sc},+)}$. That is, at time $s=k$, the second-order term cannot be neglected, and we have:
\begin{equation}
    \begin{aligned}
    &f^{(k+1)}(\bb{X}_i) = f^{(1)}(\bb{X}_i)+\sum\limits_{s=1}^{k-1}\left\langle\nabla_{\bb{\theta}} f(\bb{X}_i;\bb{\theta}^{(s)})^{(c,+)},\Delta \boldsymbol{\theta}^{(s)}\right\rangle\\&+\left\langle\nabla_{\bb{\theta}} f(\bb{X}_i;\bb{\theta}^{(k)})^{(\text{sc},+)},\Delta \boldsymbol{\theta}^{(k)}\right\rangle
    +\sum\limits_{s=1}^{k-1}o(\Vert \Delta \bb{\theta}^{(s)}\Vert_2)+\Delta^{(k)}.
\end{aligned}
\label{nomuta}
\end{equation}
where $\Delta^{(k)}$ captures the non-negligible second-order information induced by the strategy shift at step $k$.
Then starting from the $k+1$-th time step, the training of the network for this sample is guided by strategy $\nabla_{\bb{\theta}} f(\bb{X}_i;\bb{\theta}^{(s)})^{(\text{sc},+)}$. In this case, assuming strategy similarity reflects the similarity of the output $f$. Therefore we can rewrite \textbf{Eq.~\ref{nomuta}} using the shortcut strategy. That is, we merge this sample into the iterative path guided by the shortcut strategy: $f^{(k+1)}(\bb{X}_i) = f^{(1)}(\bb{X}_i)+\sum\limits_{s=1}^{k}\left\langle \nabla_{\bb{\theta}} f(\bb{X}_i;\bb{\theta}^{(s)})^{(\text{sc},+)},\Delta \boldsymbol{\theta}^{(s)}\right\rangle+\sum\limits_{s=1}^{k}o(\Vert \Delta \bb{\theta}^{(s)}\Vert_2).$
As described previously, $\Delta\bb{\theta}$ can be decomposed over the three edge sets, and thus
\begin{equation*}
    \sum\limits_{s=1}^{t-1}\left\langle\nabla_{\bb{\theta}} f(\bb{X}_i;\bb{\theta}^{(s)}),\Delta \boldsymbol{\theta}^{(s)}\right\rangle=\sum\limits_{s=1}^{t-1}\left\langle\nabla_{\bb{\theta}} f(\bb{X}_i;\bb{\theta}^{(s)}),\sum\limits_{i=1}^3\Delta \boldsymbol{\theta}^{(s)} \big|_{\mathcal{E}_i}\right\rangle. \end{equation*}
Since $\mathcal{E}_3$ is the set of edges connecting the two classes of neurons and its long-term cumulative effect on the learning of subnetworks is relatively small, we assume that \begin{equation}\frac{\sum\limits_{s=1}^{t}\left\langle\nabla_{\bb{\theta}} f(\bb{X}_i;\bb{\theta}^{(s)}),\Delta \boldsymbol{\theta}^{(s)} \big|_{\mathcal{E}_3}\right\rangle}{\sum\limits_{j=1}^{2}\sum\limits_{s=1}^{t}\left\langle\nabla_{\bb{\theta}} f(\bb{X}_i;\bb{\theta}^{(s)}),\Delta \boldsymbol{\theta}^{(s)} \big|_{\mathcal{E}_j}\right\rangle}= o(1).~\text{as}~t\rightarrow \infty
\label{Ttheta}
\end{equation}
Through the above derivations, we can explicitly write down the expression of the payoff matrix $\bb{U}^{(t)}$.
\begin{lemma}
Assume that \textbf{Assump~\ref{assu-subnetwork}} and \textbf{Assump~\ref{assump_2}} hold, together with \textbf{Eq.~\ref{Ttheta}}, 
and that $f^{(1)}(\bb{X}_i)=o(1)$. 
Then the payoff matrix $\bb{U}^{(t)}$ takes the following form:
\begin{equation}
\begin{aligned}
\bb{U}^{(t)} =& \begin{pmatrix}
  1+\gamma w_{2}^{(t)}-w_{1}^{(t)} &
  -\gamma\bigl(1+\gamma w_{2}^{(t)}-w_{1}^{(t)}\bigr) \\
  \gamma\bigl(1-w_{2}^{(t)}-\gamma w_{1}^{(t)}\bigr) &
  1-w_{2}^{(t)}-\gamma w_{1}^{(t)}
\end{pmatrix}:=\begin{pmatrix}a^{(t)} &b^{(t)}\\
c^{(t)} & d^{(t)} \end{pmatrix},
\end{aligned}
\end{equation}

where $w_i$ are the rescaled training intensities of the subnetworks $\mathcal{E}_i$.
\label{lemma-payoff}
\end{lemma}
The structure of this payoff matrix admits two Nash equilibria: one where all samples adopt the core strategy, and another where all adopt the shortcut strategy. However, these strategies may differ in nature. We observe that if the network initially learns the shortcut subnetwork, i.e., $w_2^{(t)}$
is relatively large, it may lead to $a^{(t)}>d^{(t)}$ but $a^{(t)}+b^{(t)}<c^{(t)}+d^{(t)}$. In this case, the core strategy is payoff-dominant, whereas the shortcut strategy is risk-dominant. Next, we explore how different optimization dynamics select between these two types of strategies.

\subsubsection{The Dynamics of \texorpdfstring{$z_t$}{zt}}
In this subsection, we focus on the evolution of $z_t$. Specifically, we seek to characterize how the number of individuals steering the training of each subnetwork changes during optimization, which ultimately affects the subnetwork’s final performance. As before we discussed,  at the current time step $t$, let $\bb{\Omega}_1^{(t)}$ be the set of samples adopting the core strategy, with $\vert \bb{\Omega}_1^{(t)}\vert = n_1^{(t)}$, and let $\bb{\Omega}_2^{(t)}$ be the set of samples adopting the shortcut strategy, with $\vert \bb{\Omega}_2^{(t)}\vert = n_2^{(t)}$, such that $n_1^{(t)}+n_2^{(t)}=N$. Then the parameter updating dynamics at $t$ is given by:
\begin{equation*}
\begin{aligned}
   \Delta \bb{\theta}^{(t)} 
   &= -\lambda^{(t)}\Biggl(
        \sum_{i\in\bb{\Omega}_1^{(t)}}\nabla_{\bb{\theta}}\ell(\bb{X}_i;\bb{\theta}^{(t)})
        + \sum_{j\in\bb{\Omega}_2^{(t)}}\nabla_{\bb{\theta}}\ell(\bb{X}_j;\bb{\theta}^{(t)})
      \Biggr) \\[0.5em]
   &= -\lambda^{(t)}\Biggl(
        \sum_{i\in\bb{\Omega}_1^{(t)}}\nabla_f \ell(\bb{X}_i;\bb{\theta}^{(t)})
        \,\nabla_{\bb{\theta}}f(\bb{X}_i;\bb{\theta}^{(t)}) + \sum_{j\in\bb{\Omega}_2^{(t)}}\nabla_f \ell(\bb{X}_j;\bb{\theta}^{(t)})
        \,\nabla_{\bb{\theta}}f(\bb{X}_j;\bb{\theta}^{(t)})
      \Biggr).
\end{aligned}
\end{equation*}
To make equal contribution from each individual, we denote 
\begin{equation*}
    \alpha_t = \frac{\sum\limits_{i\in\bb{\Omega}_1^{(t)}}\bigl\vert\nabla_f \ell(\bb{X}_i;\bb{\theta}^{(t)})\bigr\vert}{\sum\limits_{i\in\bb{\Omega}_1^{(t)}}\bigl\vert\nabla_f \ell(\bb{X}_i;\bb{\theta}^{(t)})\bigr\vert+\sum\limits_{j\in\bb{\Omega}_2^{(t)}}\bigl\vert\nabla_f \ell(\bb{X}_j;\bb{\theta}^{(t)})\bigr\vert}.
\end{equation*}

Let the effective number of samples adopting the core strategy at the current time step be $z_t=\alpha_t N$, and the equivalent number of samples adopting the shortcut strategy be $N-z_t$.

At the same time, we adapt the calculation for the degree of alignment between the sample's neural tangent feature and the direction of parameter evolution, performing it within the subspace where the subnetwork evolves:
\begin{equation*}
\begin{aligned}
    \left\langle \nabla_{\bb{\theta}} f^{(t)},\Delta\bb{\theta}^{(t)}\right\rangle & = \Vert\Delta\bb{\theta}^{(t)}\Vert_2\left\langle\nabla_{\bb{\theta}} f^{(t)},\frac{\sum\limits_{i=1}^3\Delta \bb{\theta}^{(t)}|_{\mathcal{E}_i}}{\Vert\Delta\bb{\theta}\Vert_2}\right\rangle\\  &\approx\Vert\Delta\bb{\theta}^{(t)}\Vert_2\left\langle\nabla_{\bb{\theta}} f^{(t)},\alpha_t\bb{V}_1^{(t)}+(1-\alpha_t)\bb{V}_2^{(t)}\right\rangle.
\end{aligned}
\end{equation*}
This demonstrates that, while the shortcut (core) strategy vector does have a positive (negative) impact on the training of the core (shortcut) sub-network, this effect is not reflected in the loss decrease at the current step due to the presence of interference weights $\bb{\theta} |_{\mathcal{E}_3}$. Superficially, the shortcut (core) feature appears to have only advanced the model along the update direction of the respective shortcut (core) sub-network. This assumption is based on empirical findings related to probing for latent features: that hidden layers encode rich, structured knowledge which is not directly used for the training task~\citep{alain2017understanding,skean2025layerlayeruncoveringhidden}.

Therefore, we analyze the evolution of $z_t$ using the following finite-difference formulation:
\begin{equation}
\begin{aligned}
\pi_A(z_t)&=\frac{z_t}{N} (1+\gamma w_2^{(t)}-w_1^{(t)})-\frac{N-z_t}{N}\cdot\gamma (1+\gamma w_2^{(t)}-w_1^{(t)})\\
\pi_B(z_t)&=\frac{z_t}{N}\cdot\gamma (1- w_2^{(t)}-\gamma w_1^{(t)})+\frac{N-z_t}{N}(1-w_2^{(t)}-\gamma w_1^{(t)}).
\end{aligned}
\label{eqn-dy}
\end{equation}
$\pi_A(z_t)$ ($\pi_B(z_t)$) denotes the expected payoff of those samples that follow the core strategy when, within the training set, $z_t$ samples are updated in the optimization direction favorable to the core (shortcut) subnetwork and the remaining $N - z_t$ samples are updated in the direction favorable to the shortcut subnetwork. This evolutionary dynamic corresponds to the training dynamic with gradient descent, since we use all $N$ data points to compute the payoff. We assume that the evolutionary dynamics of $z_t$ satisfy the following Darwinian properties (\textbf{Eq.}~\ref{darwanian}) with a mutation coefficient $\varepsilon$.

Here, $b(z_t)$ captures the evolutionary behavior of $z_t$, i.e.,a strategy that obtained a higher payoff in the previous round is reinforced in the subsequent round. The random variables $p_t$ and $q_t$ follow binomial distributions with parameter $\varepsilon$, allowing certain individuals to switch to another strategy in the next round even when the current one is superior. It shows that within a small time interval $\Delta t$, each player independently makes an incorrect move with the same small probability $\varepsilon$. The mutation coefficient is intimately linked to evolutionary stability and, in the neural-network setting, corresponds to the optimizer’s ability to escape saddle points. \textbf{Fig.~\ref{fig:3}} illustrates the overall evolutionary process during optimization.

\paragraph{Explanation of the evolution} The evolution of sample strategies during optimization can be interpreted as follows: strategies yielding higher payoffs in the current round are reinforced in subsequent iterations. A higher payoff at a given step indicates that the strategy is more task-relevant at that stage, while its reinforcement in the next step reflects the adaptive nature of the gradient. This observation is consistent with the experimental findings of~\citep{baratin2021implicit}, which demonstrate that the gradient directions of deep neural networks dynamically adapt toward task-relevant features during training. We view this phenomenon as a form of \emph{implicit regularization}: by dynamically adjusting gradient directions, the model effectively performs both feature selection and model compression.

\paragraph{Explanation of the mutation} We posit that, within evolutionary dynamics, mutation stems from the uncertainty a neural network faces when modeling the noise present in the data. As noted above, shortcut features alone cannot classify the training set; hence the network must rely on noise in the data to facilitate the learning of these shortcuts. Assuming that the noise is sample independent, the amount of mutation can be modeled by a binomial distribution. For instance, when the network intensifies training on the shortcut subnetwork edges ($\mathcal{E}_2^f$), it must also learn the noise signal from $\mathcal{E}_2^n$ to create an \emph{illusory} decline in loss, which may trap the optimizer in a local minimum. Nevertheless, the network can escape: if it fails to fit the noise of a particular sample, that sample may mutate into an individual that follows an alternative strategy. Thus, because each sample’s strategy‐mutation event is independent, and assuming the mutation probability is $\varepsilon$, with every sample at each moment having a binary outcome—mutation (1) or no mutation (0). Thus, the probability that each population undergoes a mutation follows a Bernoulli distribution. Therefore, the total number of individuals in the population that mutate follows a binomial distribution.

On the other hands, the presence of mutation noise ensures that the probability of transitioning from any other state is strictly positive. This causes the transition matrix, $P$, to be a positive matrix. Therefore a homogeneous Markov chain defined by this transition matrix is necessarily regular. This in turn guarantees that the chain is ergodic and has a unique stationary distribution, which is also the limiting distribution.

\begin{example}
     Consider a Markov chain with 3 states $\{0,1,2\}$, where for a variable $z$ and a threshold $z^*>1$, the following conditions hold
\begin{itemize}

\item[(1)] When $z>z^*$, $\pi_A(z)>\pi_B(z)$\\
\item[(2)] When $z<z^*$,
$\pi_A(z)\leqslant \pi_B(z)$.
\item[(3)] $\vert b(z)-z\vert = 1$.
\end{itemize}
Then the corresponding transition probability matrix is
\begin{equation*}
H= \begin{pmatrix}
    1-2\varepsilon-\varepsilon^2 & 2\varepsilon & \varepsilon^2 \\
    1-2\varepsilon-\varepsilon^2 & 2\varepsilon & \varepsilon^2 \\
    \varepsilon^2 & 2\varepsilon & 1-2\varepsilon-\varepsilon^2.
\end{pmatrix}
\end{equation*}

Thus, we can obtain the expression for its unique stationary distribution $\bb{\mu}$ by solving the following equations:
\begin{equation*}
\begin{aligned}
    \bb{\mu}\bb{H}&=0\\
    \bb{\mu}_1+\bb{\mu_2}+\bb{\mu}_3 &= 1.
\end{aligned}
\end{equation*}

We get $\bb{\mu}_1 = \frac{2-3\varepsilon-4\varepsilon^2}{2(1+\varepsilon)}, \bb{\mu_2}=2\varepsilon,\bb{\mu_3}=\frac{\varepsilon}{2(1+\varepsilon)}$. When $\varepsilon\rightarrow 0$, $\bb{\mu}\rightarrow (1,0,0)$.

However, when$\varepsilon\rightarrow 0$, the matrix reduces to:
\begin{equation*}
\begin{pmatrix}
    1&0&0\\
    1&0&0\\
    0&0&1\\
\end{pmatrix},
\end{equation*}
who has two stationary distribution $(1,0,0)$and $(0,0,1)$.
\end{example}
This simple example illustrates that mutation and system stability are closely related. However, when we consider the evolutionary game process of sample strategies in a neural network, the number of samples $N$ is very large. This means the dimension of the transition probability matrix is very high, making it computationally expensive to solve the equations. Furthermore, the corresponding weight vectors $\bb{w}_1$ and $\bb{w}_2$ are constantly changing, which means the transition probability matrix is also time-varying. This makes it difficult to analyze the distributional properties of the non-homogeneous Markov chain.

Therefore, we will next use the singular perturbation technique to demonstrate how we handle the evolutionary process of the two strategies when the payoff satisfies \textbf{Eq.~\ref{eqn-dy}}.

\begin{theorem}
Assume that $N$ is an even number. If $\pi_A,\pi_B$ satisfies \textbf{Eq.~\ref{eqn-dy}}, and $\forall~t$, $w_1^{(t)},w_2^{(t)}\geq0,w_1^{(t)}+w_2^{(t)}\leq 1-\beta,w_1^{(t)}-w_2^{(t)}\leq \frac{2\gamma}{1-\gamma^2}\beta$, and $0<\beta,\gamma<1$, then $\forall~z_0=z\in [N]$, $\lim\limits_{\varepsilon\rightarrow 0}\lim\limits_{t\rightarrow\infty} P^\varepsilon(z_t=0)=1$.
\label{Thm-GD}
\end{theorem}

This yields a relatively pessimistic result, indicating that gradient descent (GD) is unfavorable for learning core features. Under this setting, the shortcut strategy corresponds to the stochastically stable strategy (SSS).

However, when we introduce certain randomness into the optimization, specifically, the calculation of rewards does not entirely depend on the average of all sample decisions, we incorporate a random variable \( k_j \) that depends on the batch size \( B \). This means that we partition the $N$ samples into $\frac{N}{B}$ mini-batches. In each batch of size $B$, $k_j$ individuals adopt the core strategy ($\nabla_{\bb{\theta}} f^{c}$), while $B-k_j$ adopt the shortcut strategy ($\nabla_{\bb{\theta}} f^{sc}$). We denote it as 
\[
K\sim\text{MHG}(N,z_t,B,\frac{N}{B}).
\]

This indicates that $N$ samples are partitioned into $\frac{N}{B}$ groups, each containing $B$ samples. Among the $N$ samples, the number of those adopting the core strategy is $z_t$, and
\[
k_1+k_2+\cdots+k_{\tfrac{N}{B}}=z_t,
\]
where $k_i$ denotes the number of core-strategy samples in the $i$-th group. Consequently, the vector $\bb{K}=(k_1,k_2,\cdots, k_{\frac{N}{B}})$ follows the Multivariate Hypergeometric Distribution with the probability mass function
\[
\Pr(k_1,\ldots,k_{\tfrac{N}{B}})
= \frac{\prod_{i=1}^{N/B} \binom{B}{k_i}}{\binom{N}{z_t}},
\quad k_1+\cdots+k_{\tfrac{N}{B}}=z_t.
\]

After one epoch, the expected payoffs for core-strategy and shortcut-strategy individuals are:
\begin{equation}
\begin{aligned}
\pi_A(z_t)
&= \frac{
\bigl(1+\gamma w_2^{(t)}-w_1^{(t)}\bigr)
\sum_{j=1}^{N/B} \operatorname{I}(k_j \neq 0)\!
\left[\tfrac{k_j}{B}-\gamma\,\tfrac{B-k_j}{B}\right]
}{
\sum_{j=1}^{N/B}\operatorname{I}(k_j\neq0)
}
\\[0.6em]
\pi_B(z_t)
&= \frac{
\bigl(1-w_2^{(t)}-\gamma w_1^{(t)}\bigr)
\sum_{j=1}^{N/B} \operatorname{I}(k_j \neq B)\!
\left[\gamma\,\tfrac{k_j}{B}+ \tfrac{B-k_j}{B}\right]
}{
\sum_{j=1}^{N/B}\operatorname{I}(k_j\neq B)
}.
\end{aligned}
\label{sgd-dy}
\end{equation}

Please note that in our setting, the sample partitioning is in fact applied to the effective number of samples rather than the actual samples. Therefore, in a strict sense, the correspondence to mini-batch optimization is not entirely exact. Nevertheless, we regard this as a reasonable approximation. In addition, we assume that within a single epoch, the variations of $w_1^{(t)}$ and $w_2^{(t)}$ are negligible, such that samples from different batches within the same epoch can share the same pair $(w_1^{(t)}, w_2^{(t)})$. After that, We will give the next theorem, which indicates that in the context of stochastic optimization, the core strategy emerges as a stochastically stable strategy.
\begin{theorem}
    At every discrete time step the population is partitioned uniformly at random into $\frac{N}{ B}$ batches, and the batch size is $B$ and $B$ is devisible by $N$. Assume that $\forall~t$, $w_1^{(t)}, w_2^{(t)}\geqslant 0$, $w_1^{(t)}+w_2^{(t)}\leqslant 1$, $0<\gamma<1$. $(1+\gamma) w_2^{(t)}>(1-\gamma) w_1^{(t)}$. If the group payoff satisfies \textbf{Eq.}~(\ref{sgd-dy}), and $N\geqslant \tilde{N}$, Then we get $\lim\limits_{\varepsilon\rightarrow 0} \lim\limits_{t\rightarrow\infty}P^\varepsilon(z_t = N)=1$.
\label{the-sgd}
\end{theorem}
Although this theorem yields a positive result concerning randomness, it implies that as long as the payoff of core strategy exceeds that of the shortcut strategy, all individuals will eventually adopt core strategy. However, in practice, two issues still remain:

\paragraph{Diminishing payoff advantage of core strategy over shortcut strategy.} Even if the network has not completely learned the core strategy, the loss of samples embodying A will continue to decrease over iterations because the network is exploiting many unanticipated features (i.e., noise).

\paragraph{Irreversibility once shortcut strategy has been learned a lot.} If the network accumulates a large amount of information favoring the short strategy early on, it becomes very difficult to reverse later, since it has already moved too far down the wrong path.

Therefore, we need to establish—within the game‐theoretic framework—the temporal evolution of the training intensities of the different subnetworks. Here, we denote this by
\begin{equation*}
\bb{w}^{(t)} = (w_1^{(t)}, w_2^{(t)})
\end{equation*}

We choose to characterize it by a continuous stochastic dynamical equation and explicitly model the effect of noise.

\subsubsection{The Dynamic of $\bb{w}_t$}
\label{dy_w}
In this subsection, we will present the evolution of $\bb{w}$, that is, how the extent of updates varies across the different subnetworks.

Let \(\alpha_t:=\frac{Z_t}{N}\) denote the proportion of samples following the \emph{core}
strategy.  According to the Darwinian property (\textbf{Eq.~\ref{darwanian}}), its replicator flow is
\begin{equation*}
\begin{aligned}
d\alpha
&=\;
\begin{pmatrix}  1 & -1 \end{pmatrix}\cdot
\bb{U}\cdot
\begin{pmatrix} \alpha \\ 1-\alpha \end{pmatrix} dt+\sigma d B_t\\
&=\;
\bigl[(1+\gamma)\alpha-\gamma\bigr]
  \bigl(1+\gamma w_{2}^{(t)}-w_{1}^{(t)}\bigr)-\bigl[1-(1-\gamma)\alpha\bigr]
  \bigl(1-w_{2}^{(t)}-\gamma w_{1}^{(t)}\bigr)+\sigma dB_t.
\end{aligned}
\end{equation*}

where $\sigma dB_t$ models the noise introduced by stochastic optimization. Meanwhile, if we consider the updates of $w_1$ and $w_2$, we get
\begin{equation*}
\begin{aligned}
dw_1 &= f_1(w_1,w_2,\alpha) e^{-\tau t}dt \\
dw_2 &= f_2(w_1,w_2,\alpha) e^{-\tau t}dt \\
d\alpha_t &= b(\alpha_t) dt + \sigma dB_t + dK_{t,0} - dK_{t,1},
\end{aligned}
\end{equation*}

where
\begin{equation*}
\begin{aligned}
    f_1(w_1,w_2,\alpha)&=(1-\alpha)\gamma(1-w_2-\gamma w_1)+\alpha(1+\gamma w_2-w_1)\\
    f_2(w_2,w_1,\alpha)&=(1-\alpha)(1-w_2-\gamma w_1)-\alpha\gamma(1+\gamma w_2-w_1)\\
    b(\alpha)&=[(1+\gamma)\alpha -\gamma](1+\gamma w_2-w_1)-\\
    &~~~~~[(1-(1-\gamma)\alpha](1-w_2-\gamma w_1)].
\end{aligned}
\end{equation*}The constraint $\alpha_t \in [0,1]$ is enforced via the reflection terms $dK_{t,0}$ and $dK_{t,1}$. Note that $\alpha_t$ itself represents the proportion of effective core samples. For the sake of analytical convenience, however, we approximate it by the proportion of the actual core samples.

Note that the parameter $\tau>0$ actually models the impact of noise. The purpose of this modeling is to characterize the interference of noise with the learning of task-relevant features. As training proceeds, the model gradually strengthens its capacity to memorize noisy information, thereby diluting the effective signal strength of the true features. This phenomenon is closely related to overfitting: when the model memorizes noise rather than extracting robust representations, the learned parameters deviate from capturing the underlying task-relevant structure. $\sigma>0$ models the noise induced by random grouping. The evolution of the strategy proportion $\alpha$ depends not only on the comparison of average payoffs between strategies at the current time, but also on the randomness introduced by grouping.

Next, we use the difference in training intensities between the shortcut sub-network and the core sub-network, i.e., $w_2-w_1$, as a proxy for shortcut bias. We then examine the effects of the coefficients $\tau$ and $\sigma$ on the evolution of the equation, thereby revealing how data noise and optimization noise respectively facilitate or suppress the formation of shortcut bias.
\begin{theorem}
    For a fixed $\alpha\in(0,1)$, the term $w_2(\infty)-w_1(\infty)$ is non-decreasing with respect to $\tau$ for $\tau\in [\tau_c,\infty)$ when $w_2(0)\geqslant \frac{1-\gamma}{1+\gamma^2}$ . For a fixed $\tau$, if $$(1+\gamma)(1-w_2(t)-\gamma w_1(t))<(1-\gamma)(1+\gamma w_2(t)-w_1(t)),$$ the term $\mathbb{E}[w_2(\infty)-w_1(\infty)]$ is non-increasing with respect to $\sigma$.
\label{SDE-Thm}
\end{theorem}
\section{Experimental Verification of Theory}

\subsection{The Initial Bias}
In this subsection, we demonstrate that neural networks readily seize upon raw-data features with the greatest variance; these features are often shortcuts—for instance, background or color. Because such cues dominate the overall variance of the dataset, they become a primary source of shortcut bias during the network’s early learning phase.

We carried out the validation on a binary-class Colored-MNIST dataset. In this setting, color serves as the shortcut feature, whereas the digit identity is the core feature: digits 0–4 are labeled 0 and digits 5–9 are labeled 1. The color–label association is highly imbalanced—90\% of the samples with label 0 are red and 90\% with label 1 are green—yielding a spurious-correlation strength of 0.9. By projecting the raw data onto a two-dimensional principal-component (PC) space, we found that these two PCs largely preserve color information, and this characteristic remains after the neural-tangent feature transformation (\textbf{Fig.~\ref{Fig_4}a,c}).

In order to confirm that color‐related features account for a larger share of the data variance, we carried out a variance decomposition of the raw dataset:

\begin{equation}
\begin{aligned}
    &\sum\limits_{i} \Vert \bb{X}_i-\bb{\mu}\Vert^2 = \sum\limits_{c} N_c\Vert \bb{\mu}_c-\bb{\mu}\Vert^2+\sum\limits_{d} N_d \Vert \bb{\mu}_d-\bb{\mu}\Vert^2 \\
    &~~~~~~ +\sum\limits_{c,d}N_{cd}\Vert \bb{\mu}_{cd}-\bb{\mu}_c-\bb{\mu}_d+\bb{\mu}\Vert^2+\sum\limits_{c,d}\sum\limits_{i\in c,d}\Vert \bb{X}_i-\bb{\mu}_{cd}\Vert^2.
\label{decomposition}
\end{aligned}
\end{equation}

Here, \(c \in \{\text{Red},\,\text{Green}\}\) denotes the color category, and
\(d \in \{\le 4,\,\geqslant 5\}\) denotes the digit category.
The first term in~(\ref{decomposition}) quantifies the variance explained by the deviation of each
color‐group centroid from the overall mean, whereas the second term quantifies the variance explained by the deviation of each digit‐group centroid
from the overall mean. \(N_{\alpha}\) denotes the sample size of group \(\alpha\), whereas
\(\mu_{\alpha}\) denotes the mean of group \(\alpha\). 

We obtained \(\frac{\sum\limits_c N_c\Vert \mu_c-\mu\Vert^2}{\sum\limits_i \Vert\bb{x}_i-\mu\Vert^2} = 25.43\,\%\) and \(\frac{\sum\limits_d N_d\Vert \mu_d-\mu\Vert^2}{\sum\limits_i \Vert\bb{x}_i-\mu\Vert^2}  = 16.61\,\%\);
which indicates that the color feature accounts for a larger
proportion of the overall variance.

\subsection{The Distinction Between Full-Batch and Mini-Batch Dynamics}
Next, we will demonstrate the differences between full batch training and mini-batch training during the training process. Specifically, we focus on the changes in the tangent mapping of the samples $\boldsymbol{X} \rightarrow \nabla_{\boldsymbol{\theta}} f(\boldsymbol{X};\bb{\theta})$ throughout the training process.

We use the Color-MNIST dataset (with gaussion noise of level 0.1) and apply both full-batch gradient descent and mini-batch gradient descent (\textbf{Fig.~\ref{Fig_4}b,d}). In both cases, the training accuracy reaches 100\%, and the training loss decreases to 0. However, we observe notable differences in the dynamics of the neural tangent mapping. After training with full-batch gradient descent, we still observe a clear separation based on color in the subspace of the neural tangent projection, but the digit features are somewhat confused. In contrast, when using mini-batch gradient descent (with a batch size of 128), we observe the opposite: the neural tangent projection of the samples is now able to distinguish the categories well, but the color features become confused. This suggests that mini-batch gradient descent enables the network to capture the essential features more effectively, while full-batch gradient descent continues to train on shortcut sub-networks. The test accuracy (on i.i.d. data) of Full-Batch is 97.10\%, while the test accuracy of Small-Batch is 98.06\%. This further indicates that the noise introduced by mini-batch training effectively alleviates shortcut bias.

\subsection{Subnetwork Hypothesis}
To test the hypothesis that a neural network employs separate subnetworks to model different features, we first train a standard classifier on the Colored‑MNIST dataset to fit the overall labels. Next, with the backbone frozen, we retrain only the final layer twice—once for the digit‑classification task and once for the color‑classification task—adding an L1 penalty to the new weights. Visualising the resulting weight distributions reveals that the network relies on distinct groups of neurons to represent the two features. Our original network had 32 neurons in the last hidden layer. After applying L1-regularized fine-tuning on specific feature extraction, we found that only 4 neurons were involved in feature extraction. Among them, neurons \#1, \#10, and \#27 were responsible for modeling core features, while neuron \#14 was responsible for modeling shortcut features. This validates our hypothesis of feature sub-networks: when the number of task-relevant features is small (much fewer than the number of hidden neurons), different neurons will be responsible for modeling different features.
\begin{figure*}[!ht]
    \centering
    \includegraphics[width=1.04\linewidth]{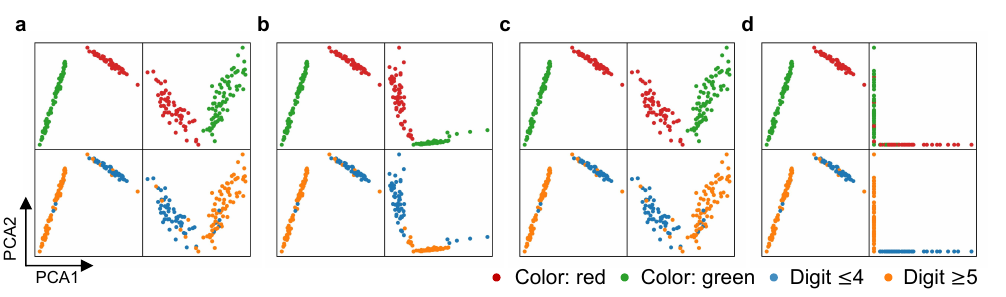}
    \caption{Simulating full-batch versus mini-batch training on the Colored MNIST dataset with a fully connected neural network. \textbf{a} and \textbf{b}, PCA visualization of the original data and the corresponding model gradients for each data point under the full-batch setting, at epoch 0 and the final epoch, respectively. \textbf{c} and \textbf{d}, similar illustration under the mini-batch setting. }
    \label{Fig_4}
\end{figure*}

\subsection{Simulation of Stochastic Differential Dynamics}
We conducted computational simulations based on the stochastic differential equation formulated in \textbf{Sec.~\ref{dy_w}}. The results highlight the roles of the data noise coefficient $(\tau)$ and the optimization noise coefficient $(\sigma)$ in shaping the training intensity of feature sub-networks. We initialized $w=(0.02,0.4)$ to highlight that the model first captures the shortcut feature, resulting in a higher initial training intensity for the shortcut sub-network. Fixing the optimization noise at $\sigma=0.4$ Comparing two levels of noise intensity ($\tau=0.3$ vs.~$\tau=0.8$), we find that with lower noise ($\tau=0.3$), the cumulative training intensity of the core sub-network steadily grows and eventually overtakes the shortcut sub-network (\textbf{Fig.~\ref{Fig_5}a}). In contrast, with higher noise ($\tau=0.8$), the core sub-network is still trained but ultimately dominated by the shortcut sub-network. These observations align with the results in \textbf{Fig.~\ref{Fig_1}a,b}. Next, we simulate the impact of the data noise coefficient $\tau$ and the optimization noise coefficient $\sigma$ on $\mathbb{E}[w_2(\infty)-w_1(\infty)]$. As $\tau$ increases (especially when $\tau$ is greater than a sufficiently large number), $\mathbb{E}[w_2(\infty)-w_1(\infty)]$ gradually becomes larger, indicating that data noise promotes the formation of shortcut bias (\textbf{Fig.~\ref{Fig_5}b}). Meanwhile, as $\sigma$ increases, $\mathbb{E}[w_2(\infty)-w_1(\infty)]$ decreases, implying that stochastic optimization can effectively mitigate shortcut bias. This is consistent with the pattern revealed in \textbf{Fig.~\ref{Fig_1}}, and it further validates our argument in \textbf{Thm.~\ref{SDE-Thm}}. 

The effect of stochastic optimization on mitigating shortcut bias is mainly reflected in its influence on the evolution of sample strategies. Fixing the data noise at $\tau=0.3$, when the stochastic optimization noise is small ($\sigma=0.1$), the sample strategies are rapidly dominated by the shortcut strategy, which is manifested by the proportion of the core strategy $\alpha$ approaching zero. In contrast, when the stochastic optimization noise is larger ($\sigma=0.4$), the sample strategies are gradually dominated by the core strategy, confirming its role as the stochastically stable solution that governs most of the training process (\textbf{Fig.~\ref{Fig_5}c}). We project the training intensities of the two feature sub-networks onto a two-dimensional plane $(w_2,w_1)$. This clearly illustrates that when the stochastic optimization noise is small, the shortcut sub-network continues to be reinforced; whereas when the stochastic optimization noise is large, there exists a certain point at which the system transitions to forgetting the shortcut sub-network, reflecting the dominance of the core strategy (\textbf{Fig.~\ref{Fig_5}d}).
\begin{figure*}[!ht]          
  \centering                 
  \includegraphics[width=0.85\linewidth]{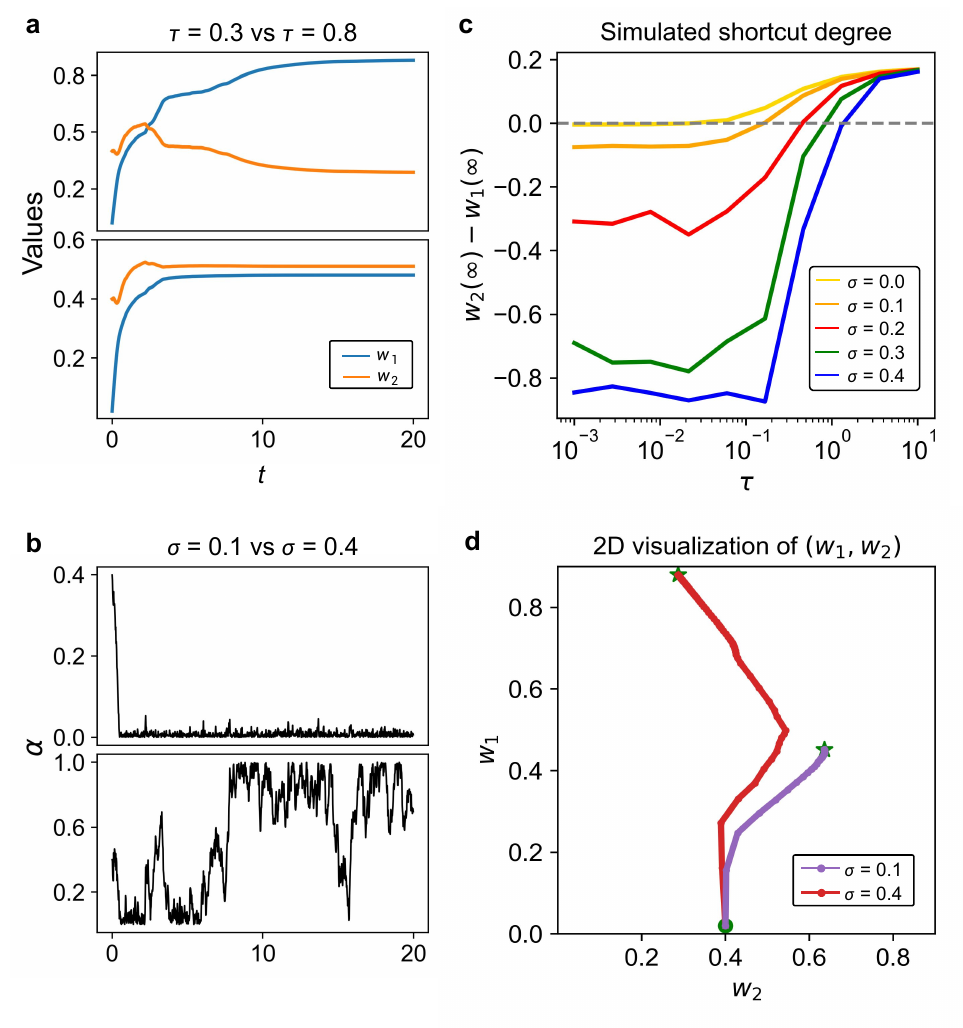} 

  \caption{Numerical simulation of the stochastic differential equation (SDE) model. 
    \textbf{a}, Evolution of the strengths of the two subnetworks ($w_1$ and $w_2$) over iterations under two different data noise levels $(\tau=0.3~\text{and}~\tau=0.8)$. 
    \textbf{b}, The proportion of core strategy evolution $(\alpha)$ as a function of optimization noise. 
    \textbf{c}, The shortcut bias, quantified by the difference in subnetwork strengths ($\mathbb{E}[w_2(\infty) - w_1(\infty)]$), is plotted against varying data noise (x-axis) and optimization noise levels (indicated by different colors).
    \textbf{d}, A 2D visualization of the optimization landscape with the strengths of the two subnetworks plotted on the two axes.}
     
  \label{Fig_5}   
\end{figure*}

\section{Discussion}
In this paper, we employ the evolutionary game theory to establish a dynamical model of shortcut feature formation and provide a mathematically precise characterization of both shortcut and core features. Our analysis shows that, from the perspective of evolutionarily stable strategies, stochastic optimization plays a direct role in mitigating shortcut bias. Furthermore, by employing continuous-time modeling with stochastic differential equations, we identify a potential mechanism through which data noise may hinder the learning of core features. This work also provides theoretical insights into techniques for mitigating shortcut bias.

\subsection{Insights for Mitigating Shortcut Bias}
From the perspective of evolutionary game theory, to make the core strategy stochastically stable, a direct approach is to increase the effective number of samples with the core strategy. In other words, by enlarging $z_t=\alpha N$, here
\[
    \alpha = \frac{\sum\limits_{i\in\Omega_1}\bigl\vert\nabla_f \ell(\bb{X}_i;\bb{\theta}^{(t)})\bigr\vert}{\sum\limits_{i\in\Omega_1}\bigl\vert\nabla_f \ell(\bb{X}_i;\bb{\theta}^{(t)})\bigr\vert+\sum\limits_{j\in\Omega_2}\bigl\vert\nabla_f \ell(\bb{X}_j;\bb{\theta}^{(t)})\bigr\vert}.
\]
The system is encouraged to reach the basin of attraction of the core strategy at a faster rate. A large body of existing techniques for mitigating shortcut bias is grounded in this perspective. For example, one line of work leverages data augmentation to artificially expand the dataset~\citep{srivastava2020robustness,yao2022improving,puli2024nuisances} or to generate counterfactual samples via causal inference~\citep{Kaushik2020Learning,zeng2020counterfactual}; another line of work adopts loss function regularization to increase the weights of hard samples~\citep{bahng2020learning,levy2020large,liu2021just,pezeshki2021gradient}.

Moreover, the reason why the core strategy can be payoff-dominant but not risk-dominant lies in the interference of the off-diagonal elements in the payoff matrix, which in this paper is reflected by the coefficient $\gamma$. In other words, samples adopting the core strategy are not restricted solely to the core sub-network, and there also exists interference between the core and shortcut sub-networks ($\mathcal{E}_3$ in \textbf{Fig.}~\ref{fig:subnetwork}). Therefore, if more disentangled representations can be learned, it may be possible for the core strategy to become both payoff-dominant and risk-dominant. Another main line of research precisely focuses on this direction. Debiasing can be achieved through task decomposition or feature subspace decomposition~\citep{zhang2021deep,lee2021learning,yang2022chroma}.

Furthermore, we can amplify the payoff gap between the core strategy and the shortcut strategy ($a$ and $d$), making the payoff advantage of the core strategy more pronounced in the evolutionary process. This, in turn, drives the stochastic dynamics to evolve in a manner dominated by the core strategy. ETF-Debias is an example of such an approach~\citep{wang2024navigate}: it employs the equiangular tight frame formed by shortcut features as a predictive representation of potential shortcut features, enabling the model to recognize earlier that shortcut features may be misleading, and thereby accelerating the learning of core feature representations. However, for scenarios where shortcut bias is unknown, this approach has its limitations. If the shortcut features cannot be explicitly specified in advance but instead must be recognized spontaneously by the model, a trade-off arises. Forcing the model to become aware that core features may be more important requires it to first learn more shortcut features. While this creates a greater game-theoretic advantage, it may also lead to irreversibility in feature learning. How to better balance this trade-off is one of our future research directions.

\subsection{Future directions}
The characteristics of shortcut learning hinder the safety, trustworthiness, and interpretability of models. This work provides an evolutionary dynamical perspective on shortcut bias. From this viewpoint, we systematically analyze how full-batch training and mini-batch training may drive the system toward different stochastically stable states, and we distinguish the roles of data noise and optimization noise through a continuous differential equation model. The evolutionary game-theoretic framework introduced in this paper is a powerful theoretical tool. However, some open questions remain. For instance, the interaction coefficient $\gamma$ influences both the interplay between strategies and the transitions of their stable states, and it is also related to the disentangled representation of features. But we have not analyzed this in detail. Moreover, our modeling is coarse-grained and does not investigate the effects of specific model components on shortcut bias. Our analytical framework is also limited to supervised learning, and whether it can be smoothly extended to other learning paradigms is another question worth exploring. Finally, how to utilize game theory to intervene in the training dynamics of models and thereby develop new algorithms to overcome shortcut bias is a very promising direction.

\newpage
\newpage 
\appendix
\section{Proofs of Theorems in the Main Text}
\textbf{Proof of Thm.~\ref{the-core}}

\begin{proof}
    Let $V^1$ be core feature, and $V^2$ be any other feature. Accoding to the difinition of the core feature, $\delta_{Y(\bb{X}_i), Y(\bb{X}_j})=\delta_{V^1(\bb{X}_i), V^1(\bb{X}_j)}$, 
    \begin{equation}
    \begin{aligned}
    S_1(f) := & \mathbb{E}\left[\delta_{f(\bb{X}_i),f(\bb{X}_j)}-\delta_{Y(\bb{X}_i),Y(\bb{X}_j)}\mid (X_i,X_j) \in \Omega_1\right]\\
    =& \mathbb{E}\left[\delta_{f(\bb{X}_i),f(\bb{X}_j)}-0\mid (X_i,X_j) \in \Omega_1\right]\\
    =& \mathbb{E}\left[\delta_{f(\bb{X}_i),f(\bb{X}_j)}\mid (X_i,X_j) \in \Omega_1\right]\geqslant 0,
    \end{aligned}
    \end{equation}
    and \begin{equation}
    \begin{aligned}
    S_2(f) := & \mathbb{E}\left[\delta_{f(\bb{X}_i),f(\bb{X}_j)}-\delta_{Y(\bb{X}_i),Y(\bb{X}_j)}\mid (X_i,X_j) \in \Omega_2\right]\\
    =& \mathbb{E}\left[\delta_{f(\bb{X}_i),f(\bb{X}_j)}-1\mid (X_i,X_j) \in \Omega_2\right]\leqslant 0.\\
    \end{aligned}
    \end{equation}
    Therefore $S_1(f)\geqslant S_2(f)$, which indicates that $V^1$ is not a shortcut feature related to $V^2$.
\end{proof}

\noindent\textbf{Proof of Thm.~\ref{the-shortcut}}

\begin{proof}
    We adopt a constructive proof by constructing $\hat{V}$ such that $\delta_{\hat{V}(\bb{X}_i),\hat{V}(\bb{X}_j)}=\delta_{f(\bb{X}_i),f(\bb{X}_j)}$. Define     \begin{equation}\begin{aligned}
    \Omega_1 &= \left\{(\bb{X}_i, \bb{X}_j) \in \mathcal{X} \times \mathcal{X} \mid \hat{V}(\bb{X}_i) \neq \hat{V}(\bb{X}_j),\ V^{\operatorname{core}}(\bb{X}_i) = V^{\operatorname{core}}(\bb{X}_j)\right\}, \\
    \Omega_2 &= \left\{(\bb{X}_i, \bb{X}_j) \in \mathcal{X} \times \mathcal{X} \mid \hat{V}(\bb{X}_i) = \hat{V}(\bb{X}_j),\ V^{\operatorname{core}}(\bb{X}_i) \neq V^{\operatorname{core}}(\bb{X}_j)\right\}.
    \end{aligned}\end{equation}
    and given that $0<\mathbb{E}\left[\delta_{f(\bb{X}_i), Y(\bb{X}_i)}\right]<1$, then $\Omega_1,\Omega_1\neq \emptyset$. Therefore we got $S_1(f)=\mathbb{E}\left[\delta_{f(\bb{X}_i),f(\bb{X}_j)}-1\mid (\bb{X}_i,\bb{X}_j)\sim\Omega_1\right]=-1$ and $S_2(f)=0$, which indicates that $\hat{V}\ll V^{\operatorname{core}}.$
\end{proof}

\noindent\textbf{Proof of Thm.~\ref{the-subspace}}

\begin{proof}
Since $\|\boldsymbol{x}_i\|_2 = 1$ for all $i$, and according to the kernel definition 
\begin{equation}
\bb{K}_{\theta} = \|\bb{x}_i\|_2 \|\bb{x}_j\|_2 \hat{h}(\left\langle \bb{x}_i, \bb{x}_j\right\rangle),
\end{equation}
where 
\begin{equation}
\hat{h}(u) = \frac{3}{4\pi}u^2 + \frac{1}{2}u + \frac{1}{2\pi},
\end{equation}
we obtain the decomposition
\begin{equation}
\begin{aligned}
\bb{K}_{\bb{\theta}} &= \frac{1}{2}\bb{K}_X + \frac{3}{4\pi}\bb{K}_X \circ \bb{K}_X + \frac{1}{2\pi}\bb{11}^\top.\\
&=\frac{1}{2}(\bb{K}_{\bb{X}}+\frac{3}{2\pi}\bb{K}_{\bb{X}}\circ\bb{K}_{\bb{X}}+\frac{1}{\pi}\bb{11}^\top).
\end{aligned}
\end{equation} Define
\(\bb{A}=\tfrac12\boldsymbol K_{\!X}\) and
\(\bb{E}=\tfrac{3}{4\pi}(\boldsymbol K_{\!X}\!\circ\!\boldsymbol K_{\!X})
 +\tfrac{1}{2\pi}\boldsymbol 1\boldsymbol 1^{\!\top}\)
and the orthogonal projectors
\(\bb{P}:=\bb{UU}^{\!\top}\) and \(\bb{P}_{\perp}:=\bb{I}-\bb{P}\), where $\bb{U} := [\bb{u}_1;\cdots;\bb{u}_r]$, and 
$\bb{u}_1,\cdots,\bb{u}_r$ are the first $r$ eigenvectors of $\bb{K}_X$.
\smallskip
The Davis–Kahan theorem states that
\begin{equation}
\|\sin\Theta(\mathcal U,\mathcal V)\|_2
\;\le\;
\frac{\|\bb{P}_{\perp}\bb{E}\,\bb{P}\|_2}{\delta}.
\end{equation}

Because the data are centered, $\sum\limits_{i=1}^N \bb{x}_i=0$, we obtain that $\forall~i$
\begin{equation}
(\bb{K}_{\bb{X}}\bb{1})_i =\sum\limits_{j=1}^N (\bb{K}_{\bb{X}})_{ij}=\sum\limits_{j=1}^N \left\langle \bb{x}_i,\bb{x}_j\right\rangle =\left\langle \bb{x}_i,\sum\limits_{j=1}^N \bb{x}_j\right\rangle = 0.
\end{equation}

Therefore,
\(\boldsymbol 1\perp\mathcal U\).
Hence
\(\bb{P}_{\perp}\boldsymbol 1\boldsymbol 1^{\!\top}\bb{P}=\boldsymbol 0\).

By the Gershgorin theorem, we get
\begin{equation}
\|\bb{P}_{\perp}(\boldsymbol K_{\!X}\!\circ\!\boldsymbol K_{\!X})\bb{P}\|_2
\;\le\;
\|\boldsymbol K_{\!X}\!\circ\!\boldsymbol K_{\!X}\|_2
\;\le\;
1+(N-1)\rho^{2}.
\end{equation}

Therefore
\begin{equation}
\|P_{\perp}E\,P\|_2
\;\le\;
\frac{3}{4\pi}\bigl[1+(N-1)\rho^{2}\bigr],
\end{equation}
and inserting this bound into (DK‐cluster) yields
\begin{equation}
\|\sin\Theta(\mathcal U,\mathcal V)\|_2
\;\le\;
\frac{\tfrac{3}{4\pi}[1+(N-1)\rho^{2}]}{\delta}
\;=\;
\frac{3}{2\pi}\,
\frac{1+(N-1)\rho^{2}}{\lambda_{r}-\lambda_{r+1}},
\end{equation}
as claimed.
\end{proof}

\noindent\textbf{Proof of Thm.~\ref{the-noise}}

\begin{proof}
The proof relies on specializing the general results for finite-rank additive perturbations of large random matrices, as established by Benaych-Georges and Nadakuditi (2011), to the specific case of a Wigner matrix. The limiting spectral distribution of the noise matrix $\bb{K}_W$ is the Wigner semicircle law, $\mu_{sc}$.

The general theory provides expressions for the emergent eigenvalues and their corresponding eigenvectors in terms of the Cauchy transform of the limiting spectral measure of the unperturbed matrix. For a measure $\mu$, the Cauchy transform is defined as $G_\mu(z) = \int \frac{1}{z-t} d\mu(t)$ for $z \notin \text{supp}(\mu)$.

According to \textbf{Thm. 2.1} of~\cite{benaych2012singular}, an eigenvalue $\beta_k$ of the perturbing matrix gives rise to an emergent eigenvalue $\rho$ of the perturbed matrix, located outside the bulk spectrum, if and only if $\beta_k$ is larger than a critical threshold. The location of this emergent eigenvalue $\rho$ is determined by the equation:
\begin{equation} \label{eq:general_eigenvalue}
\rho = G_{\mu_{sc}}^{-1}(1/\beta_k), \quad \text{or equivalently,} \quad G_{\mu_{sc}}(\rho) = \frac{1}{\beta_k}.
\end{equation}

For a Wigner matrix with entry variance $\sigma^2$, the limiting spectral measure $\mu_{sc}$ is the semicircle law supported on $[-2\sigma, 2\sigma]$. Its Cauchy transform for $z \in \mathbb{C} \setminus [-2\sigma, 2\sigma]$ is given by:
\begin{equation} \label{eq:cauchy}
G_{\mu_{sc}}(z) = \frac{z - \sqrt{z^2 - 4\sigma^2}}{2\sigma^2}.
\end{equation}
The threshold condition $\beta_k > \sigma$ ensures that a solution $\rho > 2\sigma$ exists. We substitute (\ref{eq:cauchy}) into (\ref{eq:general_eigenvalue}) to solve for $\rho$:
\begin{equation}
\frac{\rho - \sqrt{\rho^2 - 4\sigma^2}}{2\sigma^2} = \frac{1}{\beta_k}.
\end{equation}
Solving for $\rho$ gives the asymptotic location of the eigenvalue:
\begin{equation}
\rho = \beta_k + \frac{\sigma^2}{\beta_k}.
\end{equation}
This concludes the first part of the proof.

According to \textbf{Thm. 2.2} of~\cite{benaych2012singular}, the squared inner product (projection) between the eigenvector $\hat{\bb{v}}_k$ of the perturbed matrix and the original signal eigenvector $\bb{v}_k$ converges to:
\begin{equation} \label{eq:general_eigenvector}
|\langle \bb{v}_k, \hat{\bb{v}}_k \rangle|^2 \xrightarrow{a.s.} -\frac{1}{\beta_k^2 G'_{\mu_{sc}}(\rho)},
\end{equation}
where $G'_{\mu_{sc}}(\rho)$ is the derivative of the Cauchy transform evaluated at the emergent eigenvalue $\rho$.

First, we compute the derivative of the semicircle Cauchy transform (\ref{eq:cauchy}):
\begin{equation}
G'_{\mu_{sc}}(z) = \frac{d}{dz} \left( \frac{z - \sqrt{z^2 - 4\sigma^2}}{2\sigma^2} \right) = \frac{1}{2\sigma^2} \left( 1 - \frac{z}{\sqrt{z^2 - 4\sigma^2}} \right).
\end{equation}
Next, we evaluate the term $\sqrt{\rho^2 - 4\sigma^2}$ using the result from Part 1:
\begin{equation}
\rho^2 - 4\sigma^2 = \left(\beta_k + \frac{\sigma^2}{\beta_k}\right)^2 - 4\sigma^2 = \beta_k^2 + 2\sigma^2 + \frac{\sigma^4}{\beta_k^2} - 4\sigma^2 = \beta_k^2 - 2\sigma^2 + \frac{\sigma^4}{\beta_k^2} = \left(\beta_k - \frac{\sigma^2}{\beta_k}\right)^2.
\end{equation}
Since $\beta_k > \sigma$, we have $\sqrt{\rho^2 - 4\sigma^2} = \beta_k - \frac{\sigma^2}{\beta_k}$. Substituting this into the derivative expression at $z=\rho$:
\begin{equation}
G'_{\mu_{sc}}(\rho) = \frac{1}{2\sigma^2} \left( 1 - \frac{\beta_k + \sigma^2/\beta_k}{\beta_k - \sigma^2/\beta_k} \right) = \frac{1}{2\sigma^2} \left( \frac{(\beta_k - \sigma^2/\beta_k) - (\beta_k + \sigma^2/\beta_k)}{\beta_k - \sigma^2/\beta_k} \right)
\end{equation}
\begin{equation}
= \frac{1}{2\sigma^2} \left( \frac{-2\sigma^2/\beta_k}{(\beta_k^2 - \sigma^2)/\beta_k} \right) = \frac{-1}{\beta_k^2 - \sigma^2}.
\end{equation}
Finally, we substitute this result back into the general formula (\ref{eq:general_eigenvector}):
\begin{equation}
|\langle \bb{v}_k, \hat{\bb{v}}_k \rangle|^2 \xrightarrow{a.s.} -\frac{1}{\beta_k^2 \left( \frac{-1}{\beta_k^2 - \sigma^2} \right)} = \frac{\beta_k^2 - \sigma^2}{\beta_k^2} = 1 - \frac{\sigma^2}{\beta_k^2}.
\end{equation}
This completes the proof.
\end{proof}
\noindent\textbf{Proof of Lem.~\ref{lemma-payoff}}
\begin{proof}
    Now, we will quantify the mismatch between the sample's strategy $\nabla_{\bb{\theta}} f$ and the network's update direction $\Delta\bb{\theta}$. For instance, consider a class-$1$ sample that adopts the \emph{core} tangent feature
$\nabla_{\bb{\theta}} f^{(\text{c},+)}$ to steer learning,
while the update is in fact dominated by the \emph{shortcut} subnetwork, i.e, $\Delta\bb{\theta}^{(t-1)} = \eta^{(t-1)} \bb{V}_1^{(t-1)}$.
Its instantaneous payoff at step~$t$ is then approximated by
\begin{equation}
\begin{aligned}
\operatorname{payoff}^{(t)}(\bb{X}_i)
&= \frac{1}{2N}\Bigl(y_i - f^{(t-1)}(\bb{X}_i;\bb{\theta}_i^{(t-1)})\Bigr)^2 \\
&\quad - \frac{1}{2N}\Bigl(
    y_i - f^{(t-1)}(\bb{X}_i;\bb{\theta}^{(t-1)})
    - \eta^{(t-1)} \bigl\langle
      \nabla_{\bb{\theta}} f(\bb{X}_i;\bb{\theta}^{(t-1)})^{(\text{c},+)};
       \bb{V}_1^{(t-1)}
     \bigr\rangle
     + o(\eta^{(t-1)})
  \Bigr)^2 \\
&= \frac{\eta^{(t-1)}}{N}
   \Bigl\langle
     \bb{V}_1^{(t)},
    \nabla_{\bb{\theta}} f(\bb{X}_i;\bb{\theta}^{(t-1)})^{(\text{c},+)}
   \Bigr\rangle
   \Bigl(1 - f^{(t-1)}(\bb{X}_i)\Bigr)
   + o(\eta^{(t-1)}) \\
&= -\frac{\eta^{(t-1)}}{N}\,\frac{\gamma M}{\sqrt{1+\gamma^2}}
   \Bigl(1 - f^{(t-1)}(\bb{X}_i)\Bigr)
   + o(\eta^{(t-1)}) \, .
\end{aligned}
\end{equation}
Expanding $f^{(t-1)}(\bb{X}_i)$ we get
\begin{equation}
\begin{aligned}
    \operatorname{payoff}^{(t)}(\bb{X}_i)
  \;&=\; -\frac{ \eta^{(t-1)}}{N}\frac{\gamma M}{\sqrt{1+\gamma^2}}(1-f^{(1)}(\bb{X}_i)-\sum\limits_{s=1}^{t-2}\left\langle \nabla_{\bb{\theta}} f(\bb{X}_i;\bb{\theta}^{(s)})^{(\text{c},+)},\Delta \bb{\theta}^{(s)}\mid_{\mathcal{E}_1\cup \mathcal{E}_2}\right\rangle)\\
  &~~~~~~+o(\sum\limits_{s=1}^{t-1}  \eta^{(s)})\\
  &=-\frac{\eta^{(t-1)}}{N}\frac{\gamma M}{\sqrt{1+\gamma^2}}(1-f^{(1)}(\bb{X}_i)+\frac{\gamma}{\sqrt{1+\gamma^2}}M w_2^{(t)}-\frac{1}{\sqrt{1+\gamma^2}}Mw_1^{(t)})\\
  &~~~~~~+o(\sum\limits_{s=1}^{t-1}\eta^{(s)}).
\end{aligned}
\end{equation}



Suppose $f^{(1)}=o(1)$, and omitting the constant term  $\frac{\eta^{(t-1)}M}{N\sqrt{1+\gamma^2}}$,  carrying out the same calculation for every strategy–subnetwork pair yields the
following payoff matrix:
\begin{equation}
\begin{aligned}
&
~~~~~\begin{pmatrix}
    \operatorname{pay-off}\left(\left\langle\nabla_{\bb{\theta}}f^c(\bb{X}_i;\bb{\theta}^{(t)}),\Delta \boldsymbol{\theta}^{(\operatorname{c},t)}\right\rangle,\cdot\right) & \operatorname{pay-off}\left(\left\langle \nabla_{\bb{\theta}}f^c(\bb{X}_i;\bb{\theta}^{(t)}),\Delta \boldsymbol{\theta}^{(\operatorname{sc},t)}\right\rangle,\cdot\right) \\
    \operatorname{pay-off}\left(\left\langle\nabla_{\bb{\theta}}f^{\operatorname{sc}}(\bb{X}_i;\bb{\theta}^{(t)}),\Delta \boldsymbol{\theta}^{(\operatorname{c},t)}\right\rangle,\cdot\right) & \operatorname{pay-off}\left(\left\langle\nabla_{\bb{\theta}}f^{sc}(\bb{X}_i;\bb{\theta}^{(t)}),\Delta \boldsymbol{\theta}^{(\operatorname{sc},t)}\right\rangle,\cdot\right)
\end{pmatrix}\\
&=\begin{pmatrix}
  1+\gamma w_{2}^{(t)}-w_{1}^{(t)} &
  -\gamma\bigl(1+\gamma w_{2}^{(t)}-w_{1}^{(t)}\bigr) \\
  \gamma\bigl(1-w_{2}^{(t)}-\gamma w_{1}^{(t)}\bigr) &
  1-w_{2}^{(t)}-\gamma w_{1}^{(t)}
\end{pmatrix},
\end{aligned}
\end{equation}
where we rewrite $w_1^{(t)},w_2^{(t)}$ as $w_1^{(t)}:=\frac{1}{\sqrt{1+\gamma^2}}M w_1^{(t)}$, $w_2^{(t)}:=\frac{1}{\sqrt{1+\gamma^2}}M w_2^{(t)}$.
\end{proof}
\noindent\textbf{Proof of Thm.~\ref{Thm-GD}}

\begin{proof} Consider the Darwinian property $\text{sign}(b(z)-z)=\text{sign}(\pi_A(z)-\pi_B(z))$, and the mutation $z_{t+1}=b(z_t)+x_t-y_t,~x_t\sim \operatorname{Bin}(N-b(z_t),\varepsilon)$ and $y_t\sim \operatorname{Bin}(b(z_t),\varepsilon)$. The dynamical system defines a time-inhomogeneous Markov process:
\begin{equation}
p_{ij}^{(t)} = \mathbb{P}(z_{t+1}=j\mid z_t=i).
\end{equation}
with all elements in the matrix $\bb{P}^{(t)}=[p_{ij}^{(t)}]$ are positive. Therefore for a fixed $\bb{P}$, the Markov chain has a unique stationary distribution, and it has the stability:
\begin{equation}
\forall ~\bb{\mu}\in\Delta^N, \bb{\mu}\bb{P}^t\rightarrow \bb{\mu}^*,~\text{as}~t\rightarrow \infty.
\end{equation}
where $\Delta^N =\{(\mu_0,\mu_1,\cdots, \mu_{N-1}, \mu_N)\mid \sum\limits_{i=0}^N \mu_i = 1, \mu_i\geqslant 0\}$.
Denote the set of long run equilibria is $C(\bb{\mu}^*)=\{i\in Z\mid \mu^*_i>0\}.$
\\

For each state $z\in [N]$, we construct a tree with $z$ as the unique root node and for each edge \( j\rightarrow i \) on the tree, we define its weight as \( p_{ij} \).   We call it as $z-$tree. Denote the set of all $z$-tree by $\mathcal{H}_z$. Then we take the summation of the products of the edge weight for all $z-$trees:

\begin{equation}
q_z = \sum\limits_{T\in \mathcal{H}_z} \prod_{(j\rightarrow i)\in T}~p_{ij}.
\end{equation}
Let $\bb{q}$ be the vector, where $\bb{q} \equiv (q_0,\cdots, q_N)$. We can show that $\bb{q}$ is proportional to $\mu^*$ (\textbf{Lem.~\ref{prop}}). Therefore $\mu_z^*=\lim\limits_{\varepsilon\rightarrow 0}\lim\limits_{t\rightarrow \infty}  \mu_z=\lim\limits_{\varepsilon\rightarrow 0}\frac{q_z(\varepsilon)}{\sum\limits_{i} q_i(\varepsilon)}$. If $\forall z, \lim\limits_{\varepsilon\rightarrow 0}q_z(\varepsilon)=0$, then for $z\in C(\mu^*)$, i.e., $\mu_{z^*} > 0$, $q_{z^*}$ converges to zero at the lowest rate. 

Thus, the problem is now reduced to analyzing the order of convergence of $q(z)$ for different $z$. Let  $q_z=\Theta(\varepsilon ^{E_{\mathcal{H}_z}})$ and $p_{ij}=\Theta(\varepsilon^{E_{ij}})$. We denote the exponential term $E_{ij}$ as the mutation energy of the edge $(j\rightarrow i)$, and $E_{\mathcal{H}_z}$ as the mutation energy of $\mathcal{H}_z$, i.e., the infimum of the mutation energies of all $z-$trees. We are given the quantity $q_z$ as a sum of terms with different orders of magnitude
\begin{equation}
q_z = \sum\limits_{T\in \mathcal{H}_z} \Theta \left(\varepsilon^{\sum\limits_{(j\rightarrow i)\in T}E_{ij}}\right).
\end{equation}
As $\varepsilon\rightarrow 0$, the asymptotic behavior of the sum is dominated by the term with the smallest exponent. It therefore follows that :
\begin{equation}
    E_{\mathcal{H}_z} = \min\limits_{T\in \mathcal{H}_z} E_T = \min\limits_{T\in \mathcal{H}_z}\sum\limits_{(j\rightarrow i)\in T} E_{ij}.
\end{equation}

We then claim that the mutation energy of the edge $E_{ij}=\vert b(i)-j\vert$ (\textbf{Lem.~\ref{weight}}).

Now we calculate \(\pi_A\) and \(\pi_B\) according to equation~(\ref{eqn-dy}), and let \(z^*\) satisfy \(\pi_A = \pi_B\), then \begin{equation}z^* = N\frac{1}{1+\frac{(1+\gamma  w_2^{(t)}-w_1^{(t)})-\gamma(1-w_2^{(t)}-\gamma w_1^{(t)})}{1-w_2^{(t)}-\gamma w_1^{(t)}+\gamma (1+\gamma w_2^{(t)}-w_1^{(t)})}}.\end{equation}   If $w_2^{(t)}-w_1^{(t)}\leqslant \frac{2\gamma}{1-\gamma^2}\beta$ and $w_2^{(t)}+w_1^{(t)}\leqslant 1-\beta$, $w_2^{(t)},w_1^{(t)}\geqslant 0,0<\beta,\gamma<1$, then \begin{equation}\frac{(1+\gamma  w_2^{(t)}-w_1^{(t)})-\gamma(1-w_2^{(t)}-\gamma w_1^{(t)})}{1-w_2^{(t)}-\gamma w_1^{(t)}+\gamma (1+\gamma w_2^{(t)}-w_1^{(t)})}<1,~z^*>\frac{N}{2}.\end{equation} 
If $z>z^*$, then $\pi_A(z)>\pi_B(z)$. If $z<z^*$, then $\pi_A(z)<\pi_B(z)$. We denote that
\begin{equation}
\chi_1=\min\{z\in Z\mid \pi_A(z)>\pi_B(z)\},\chi_2=\max\{z\in Z\mid \pi_A(z)<\pi_B(z)\},
\end{equation} and then $\frac{N}{2}\leqslant \chi_2<z^*<\chi_1\leqslant N$. The next step is to calculate $E_{\mathcal{H}_z}$ for all values of $z$.

If $z = 0$, consider an arbitrary $0-tree$ $T\in \mathcal{H}_0$. The mutation energy of this tree, $E_T$, satisfies $E_T\geqslant E_h$, where $h$ is the path from node $0$ to node $N$. Note that for any node $z^\prime \leqslant \chi_2$, we can always find a large enough integer $M(z^\prime)$ for which $b^{M(z^\prime)}(z^\prime) = 0$, where $b^M(\cdot)$ is the $M$-fold iteration of $b(\cdot)$. This means there is always a path $h^\prime$ from $0$ to $z^\prime$ along which all edge weights are $\Theta(1)$, i.e., for any $(j\rightarrow i)\in h^\prime$, $E_{ij}=0$:
\begin{equation}\begin{tikzpicture}[
    >=Stealth,           
    node distance=1.0cm,   
    level distance=1.0cm 
  ]

  \node (n0)    {$b^{M(z^\prime)}M(z^\prime)=0$};
  \node (dots)    [right=of n0]    {$\cdots$};
  \node (b-2)  [right=of dots]    {$b^2(z^\prime)$};
  \node (b-1)  [right=of b-2]  {$b(z^\prime)$};
  \node (zprime)  [right=of b-1]  {$z^\prime$};

  \draw[->] (n0)   -- (dots);
  \draw[->] (dots)   -- (b-2);
  \draw[->] (b-2) -- (b-1);
  \draw[->] (b-1) -- (zprime);
\end{tikzpicture}\end{equation}

Thus the lower bound of the mutation energy for all paths from $0$ to $N$ is achieved by paths of the following type:
\begin{equation}
    \begin{tikzpicture}[
    >=Stealth,           
    node distance=1.0cm,   
    level distance=1.0cm 
  ]

  \node (n0)    {$b^{M(z^\prime)}(z^\prime)=0$};
  \node (dots)    [right=of n0]    {$\cdots$};
  \node (b-2)  [right=of dots]    {$b^2(\chi_2)$};
  \node (b-1)  [right=of b-2]  {$b(\chi_2)$};
  \node (zprime)  [right=of b-1]  {$\chi_2$};
  \node (N)  [right=of zprime]  {$N$};

  \draw[->] (n0)   -- (dots);
  \draw[->] (dots)   -- (b-2);
  \draw[->] (b-2) -- (b-1);
  \draw[->] (b-1) -- (zprime);
  \draw[dashed, ->] (zprime) -- (N);
\end{tikzpicture}
\label{path-1}
\end{equation}

Thus $E_T\geqslant N-\chi_2$, and $E_{\mathcal{H}_0}=\min\limits_{T\in \mathcal{H}_0} E_T\geqslant N-\chi_2$.

In fact, we can indeed find such a $0-$tree whose mutation energy is $N-\chi_2$. It is constructed by augmenting the paths in type of~(\ref{path-1}). Specifically, for every node $z>\chi_1$ and $z\neq N$, an edge is drawn directly from node $\chi_1$ to node $z$. This shows that
\begin{equation}
E_{\mathcal{H}_0} = \min\limits_{T\in \mathcal{H}_0} E_T\leqslant N-\chi_1.
\end{equation}
Therefore $E_{\mathcal{H}_0}=N-\chi_2$.

Following a symmetrical argument, it can be shown that $E_{\mathcal{H}_N}=\chi_1$.

Next we consider $E_{\mathcal{H}_z}$ in the case where $0<z<N$. For $0<z\leqslant \chi_2$, note that for any $z^\prime>z$, any path from $z$ to $z^\prime$ can be extended to a path from 0 to $z^\prime$ without changing the mutation energy of the path:
\begin{equation}
    \begin{tikzpicture}[
    >=Stealth,           
    node distance=1.0cm,   
    level distance=1.0cm 
  ]
  \node (0) {$b^{M(z)}(z)=0$};
  \node (n1) [right=of 0]   {$b^{M(z)-1}(z)$};
  \node (dots)    [right=of n1]    {$\cdots$};
  \node (n2)  [right=of dots]    {$b^1(z)$};
  \node (n3)  [right=of n2]  {$z$};
  \node (dots2)  [right=of n3]  {$\dots$};
  \node (zprime)  [right=of dots2]  {$z^\prime$};

  \draw[dashed, ->] (0)   -- (n1);
  \draw[dashed, ->] (n1)   -- (dots);
  \draw[dashed, ->] (dots) -- (n2);
  \draw[dashed, ->] (n2) -- (n3);
  \draw[->] (n3) -- (dots2);
  \draw[->] (dots2) -- (zprime);
\end{tikzpicture}
\label{path-2}
\end{equation}

However, for any $z^{\prime\prime}<z$, the mutation energy from $z$ to $z^{\prime\prime}$ is necessarily greater than 0. This implies for any $T_z\in \mathcal{H}_z$, we can find a $T_0\in \mathcal{H}_0$ which has a lower mutation than $T_z$. Therefore, $E_{\mathcal{H}(z)}>E_{\mathcal{H}(0)}=N-\chi_1$.

Following a symmetrical argument, it can be shown that for any $\chi_2<z<N$,$E_{\mathcal{H}_z}> E_{H_N}=\chi_1$.

Overall, we get the result below
\begin{equation}
    \begin{cases}
        E_{\mathcal{H}_0}&= N-\chi_2\\
        E_{\mathcal{H}_N}&= \chi_1\\
        E_{\mathcal{H}_z}&>N-\chi_2,~0<z\leqslant \chi_2\\
        E_{\mathcal{H}_z}&>\chi_1,~\chi_2<z\leqslant N.\\   
    \end{cases}
\end{equation}
Due to $N-\chi_2<\chi_1$, $E_{\mathcal{H}_0}$ is minimum among all $E_{\mathcal{H}_z}$, which means that $q_0(\varepsilon)$ converges to zero at the lowest rate, and $0\in C(\bb{\mu}^*)$.

Assume now that \( w_2^{(t)} \) and \( w_1^{(t)} \) are time-varying. This generates a time-inhomogeneous Markov process at each time \( t \), but the transition matrices can only take finitely many possible forms. We denote by \( \mathcal{P} = \{P_1, \dots, P_K\} \) the set of all admissible transition probability matrices. 

According to \textbf{Lem.}~\ref{lemma-switching}, this time-inhomogeneous Markov chain satisfies the weak ergodicity. The existence of a unique stationary distribution then implies strong ergodicity, which in turn ensures the convergence to a unique limiting distribution regardless of the initial state, 
which implies
$\lim_{\varepsilon\to 0}\,\lim_{N\to\infty}
    \bb{z}\,Q_{0}(\varepsilon)Q_{1}(\varepsilon)\cdots Q_{N-1}(\varepsilon)
    \;=\;
    \bb{z}^{*}
  \;
  (\forall \bb{z}\in\Delta_{m}),$ where $\bb{z}^*$ put the only positive probability on $z = N$. So $\forall~z_0=z\in [N],\lim\limits_{t\rightarrow \infty}\lim\limits_{\varepsilon\rightarrow 0}P^\varepsilon(z_t=N)=1$.

\end{proof} 

\noindent\textbf{Proof of Thm.~\ref{the-sgd}}

\begin{proof} We still denote the state as $z_t \in \Omega = \{0,1,2,\dots,N\}$. The evolution of $z_t$ follows Darwinian properties (\textbf{Eq.~\ref{darwanian}} in the main text).

We denote by $\bb{P}^{(t)}$ the transition probability matrix of $z_t$ at time $t$, and let $\bb{Q}^{(t)}$ be the corresponding transition probability matrix without the mutation coefficient. This induces a non-time-homogeneous Markov process. Similarly, we first assume that $\bb{P}^t$ does not depend on time $t$, meaning $w_2, w_1$ remain constant with respect to $t$. We analyze the convergence properties of this time-homogeneous Markov chain and then extend the analysis to the convergence of the non-time-homogeneous Markov chain. For notational convenience, we denote 

\begin{equation*}
\begin{pmatrix}
    1+\gamma w_2-w_1 & -\gamma(1+\gamma w_2-w_1)\\
    \gamma (1- w_2-\gamma w_1) & 1-w_2-\gamma w_1
\end{pmatrix}:=\begin{pmatrix}
    a&b\\c& d
\end{pmatrix}.
\end{equation*}
If $w_2,w_1,\gamma$ satisfy the conditions of the theorem, then we have $a> d$. Note that $N$ is the absorbing state of $\bb{Q}$ in that $\forall z\in \Omega$, $\bb{Q}_{N,z}=0$. Define its basin of attraction undet $\bb{Q}$ satisfies
\begin{equation}
A_{\bb{Q}}(N)\equiv \{z\in\Omega\mid \bb{Q}^n(z,N)>0, ~\text{for some}~n=1,2,\cdots\}.
\end{equation}

Let $M = \frac{N}{B}$, at any time $t$, the random vector $\boldsymbol{K} = (k_{1},\dots,k_{M})$ follows the Multivariate Hypergeometric distribution $\text{MHG}(N, z_{t}, B, \frac{N}{B})$.

We claim that for every $z_t \geqslant  \lceil\tau B\rceil$ where $\tau  =\frac{a-b}{\,a-d\,}\geqslant 0$ (\textbf{Lem.}~\ref{lemma-Realbasin}), one can find a realization of the multivariate hypergeometric distribution $H(B;\,N, Z_{t})$ for which $\pi_{A} \;>\; \pi_{B}.$ And if $z_t\geqslant \tau B$, then there is positive probability that $z_{t+1}>z_t.$ Therefore,
\begin{equation}
\{z_t\geqslant \lceil \tau B\rceil\mid ~\tau = \frac{a-b}{a-d}\}\subset A_Q(N).
\end{equation}


On the other hand, for any given $L$, there exists

\begin{equation}
\tilde{N} = \frac{L(a-b)-B\lfloor\frac{L}{B}\rfloor d}{a-d},
\end{equation}
such that when $N \geqslant \tilde{N}$, it holds that for any realization of the multivariate hypergeometric distribution $\textbf{MHG}(B;N, Z_t)$ (\textbf{Lem.~\ref{lemma-Realbasin}}),

\begin{equation}
\pi_A(z_t) > \pi_B(z_t), \quad \forall z_t \in \{N - L, N - L + 1, \cdots, N - 1\}.
\end{equation}

Therefore, for $L=\lceil\tau B\rceil$, we can find a sufficiently large $N>\tilde{N}$ such that $\{z_t\geqslant L\}\subset A_Q(N)$.

It shows that if $N>\tilde{N}$, the length of $A_Q(N)$ is at least $N-L$, while the length of $A_Q(0)$ is at most $N-L-1$. Then, by applying the same technique used in the proof of \textbf{Thm~\ref{Thm-GD}}, we can get 

\begin{equation}
    \begin{cases}
        E_{\mathcal{H}_N}&\leqslant \lceil\tau B\rceil\\
        E_{\mathcal{H}_0}&\geqslant \lceil\tau B\rceil\ +1\\
        E_{\mathcal{H}_z}&>E_{\mathcal{H}_0},~0<z\leqslant A_{\bb{Q}}(0)\\
        E_{\mathcal{H}_0}&>E_{\mathcal{H}_N},~A_{\bb{Q}}(0)<z< N.\\
        \end{cases}
\end{equation}

Therefore $E_{\mathcal{H}_N}$ is minimum among all $E_{\mathcal{H}_z}$, which means that $q_N(\varepsilon)$ converges to 0 at the lowest rate, and $N\in C(\mu_*)$.
\end{proof}

\noindent\textbf{Proof of Thm.~\ref{SDE-Thm}}

\begin{proof}
For a fixed $\alpha$, the system reduces to an ordinary differential equation. The dynamics of $\bb{w}=(w_1,w_2)^\top$ are
    \begin{equation}
    \begin{aligned}
    \frac{d w_1}{d t} &= [(1-\alpha)\gamma(1- w_2-\gamma w_1)+\alpha(1+\gamma w_2-w_1)]\cdot e^{-\beta(t)},\\
    \frac{d w_2}{d t} &= [(1-\alpha)(1-w_2-\gamma w_1)-\alpha\gamma (1+\gamma w_2- w_1)]\cdot e^{-\beta(t).} \\
    \end{aligned}
    \end{equation}

Define $s(t) = \int_0^t e^{-\tau u} du = \frac{1}{\tau}(1-e^{-\tau t})$, so that $\frac{ds}{dt}=e^{-\tau t}$. By the chain rule,
\begin{equation}
\frac{d\bb{w}}{dt}=\frac{d\bb{w}}{ds}\frac{ds}{dt} \Rightarrow  \frac{dw}{ds} =\bb{M}\bb{w}+b.
\end{equation} Hence, $\frac{d\bb{w}}{ds} = f(\bb{w}),$ which is an autonomous system, where \begin{equation}\bb{M} = \begin{pmatrix}-\gamma^2(1-\alpha)-\alpha & -\gamma(1-\alpha)+\alpha\gamma\\
-\gamma(1-\alpha)+\alpha\gamma & -(1-\alpha)-\alpha\gamma^2\end{pmatrix},~~~~\bb{b} = \begin{pmatrix}
    (1-\alpha)\gamma +\alpha\\
    (1-\alpha)-\alpha\gamma
\end{pmatrix}.\end{equation}
The equilibrium is 
\begin{equation}
    \bb{w}_{\text{eq}}:=-\bb{M}^{-1} \bb{b}=\begin{pmatrix}
        \frac{1+\gamma}{1+\gamma^2}\\
        \frac{1-\gamma}{1+\gamma^2}
    \end{pmatrix}.
\end{equation}

Thus the solution is 
\begin{equation}
    \bb{w}(c)=\bb{w}_{\text{eq}}+\exp(\bb{M} s)(\bb{w}(0)-\bb{w}_{\text{eq})}.
\end{equation}

$\bb{M}$ is symmetric and admits the orthonormal eigendecomposition

\begin{equation}
    \bb{M}=\bb{Q}\operatorname{diag}(\lambda_1,\lambda_2)\bb{Q}^\top, \lambda_1 = -(1-\alpha)(1+\gamma^2), \lambda_2=-\alpha(1+\gamma^2),
\end{equation}
with 
\begin{equation}
    u_1 = \frac{1}{\sqrt{1+\gamma^2}}\begin{pmatrix}
        \gamma\\1
    \end{pmatrix},~~u_2=\frac{1}{\sqrt{1+\gamma^2}}\begin{pmatrix}
        1\\-\gamma
    \end{pmatrix},~~\bb{Q}=[u_1,u_2].
\end{equation}

Let $r:=\bb{w}(0)-\bb{w}_{\text{eq}}$. Define $c^\top=[-1,1]$, and set
\begin{equation}
    \bb{a}^\top :=\bb{c}^\top\bb{Q}=(a_1,a_2)=\frac{1}{\sqrt{1+\gamma^2}}(1-\gamma,-(1+\gamma)), ~~~b:=\bb{Q}^\top \bb{r}=(b_1,b_2)^\top, 
\end{equation}
with $b_1=\frac{1}{\sqrt{1+\gamma^2}}(\gamma(w_1(0)-w_{1,\text{eq}})+(w_2(0)-w_{2,\text{eq}})),~b_2=\frac{1}{\sqrt{1+\gamma^2}}((w_1(0)-w_{1,\text{eq}})-\gamma(w_2(0)-w_{2,\text{eq}}))$.

Then the long-run limit $\bb{w}^*(\tau)=\bb{w}(s_{\infty})$ with $s_{\infty}=\frac{1}{\tau}$ yields

\begin{equation}
    w_2^*(\tau)-w_1^*(\tau)=\bb{c}^\top \bb{w}_{\text{eq}}+a_1b_1 e^{\lambda_1/\tau}+a_2b_2e^{\lambda_2/\tau}.
\end{equation}

Consequently,
\[
\frac{d}{d\tau}\bigl(w_2^*(\tau)-w_1^*(\tau)\bigr)\ \text{has the same sign as}\ 
N(\tau):=a_1 b_1\lambda_1 e^{\lambda_1/\tau}+a_2 b_2\lambda_2 e^{\lambda_2/\tau}.
\]
If $-\frac{a_2 b_2 \lambda_2}{a_1 b_1 \lambda_1}>0$, there exists a unique
\[
\quad
\tau_c \;=\; \frac{\lambda_1-\lambda_2}{\displaystyle \log\!\Bigl(-\frac{a_2 b_2 \lambda_2}{a_1 b_1 \lambda_1}\Bigr)}
\;=\; \frac{(2\alpha-1)(1+\gamma^2)}{\displaystyle \log\!\Bigl(-\frac{a_2 b_2 \lambda_2}{a_1 b_1 \lambda_1}\Bigr)}
\quad
\]
such that $N(\tau)<0$ for $\tau<\tau_c$ and $N(\tau)>0$ for $\tau>\tau_c$. Therefore, $w_2(\infty)-w_1(\infty)$ is non-decreasing with respect to $\tau$ for $\tau\in[\tau_c,\infty)$.

Now we figure how $\mathbb{E}[w_1(\infty)-w_2(\infty)]$ depends on $\sigma$ by first decomposing $\mathbb{E}[w_1(\infty)-w_2(\infty)]$:
\begin{equation}
\mathbb{E}[w_1(\infty)-w_2(\infty)]=\int_0^\infty \mathbb{E}[f_1(w_t, \alpha_t)-f_2(w_t, \alpha_t)]\exp(-\tau t) dt
\label{w_1infty}.
\end{equation}

Let $F(w,\alpha):=f_1(w,\alpha)-f_2(w,\alpha)$, and we get
\begin{equation}
\begin{aligned}
F(w,\alpha) &:= [(1-\gamma)(1-w_2-\gamma w_1)+(1+\gamma)(1+\gamma w_2-w_1)] \alpha\\
&~~~~~~-\gamma(1+\gamma w_2-w_1)-(1-w_2-\gamma w_1)\\
&:=m(w)\alpha+c(w).
\end{aligned}
\end{equation}
where $m(w)=(1+\gamma)(1+\gamma w_2-w_1)+(1-\gamma)(1-w_2-\gamma w_1)$, $c(w) = -\gamma(1+\gamma w_2-w_1)-(1-w_2-\gamma w_1).$

If $w_2,w_1$ satisfies the prescribed constraints, then $m(w)\geqslant 0$. Futhermore, the dynamic of $\alpha_t$ satisfies
\begin{equation}
    d(\alpha_t)=(m(w_t)\alpha_t + c(w_t))dt + \sigma dB_t+ d K_t.
\end{equation}

For an SDE with reflecting boundaries, the corresponding stationary Fokker-Planck equation requires the probability flux, $J(\alpha)$, to be zero at the boundaries. In the stationary state, the flux must be constant throughout the domain, which implies $J(\alpha)\equiv 0$, for all $\alpha\in [0.1]$. The zero-flux condition is 
\begin{equation}
    J(\alpha)=b(\alpha)p(\alpha;\sigma)-\frac{\sigma^2}{2}\frac{d}{d\alpha}p(\alpha;\sigma)=0.
\end{equation}

This is a first-order separable ordinary differential equation for $p(\alpha;\sigma)$, which solves to
\begin{equation}
    p(\alpha,\sigma)=N(\sigma)\exp\left(\frac{2}{\sigma^2}\int_0^\alpha b(y) dy\right).
\end{equation}

Let $G(\alpha)=\int_0^\alpha b(y) dy$, then we obtain the precise form of the stationary PDF of $p(\alpha;\sigma)$:
\begin{equation}
p(\alpha,\sigma)=N(\sigma)\exp\left(\frac{2}{\sigma^2}G(\alpha)\right).
\end{equation}

The stationary expectation is defined as $\bar{\alpha}(\sigma)=\int_0^1\alpha p(\alpha;\sigma)$, and we differentiate with respect to $\sigma$. A rigorous calculation yields:
\begin{equation}
    \frac{d\bar{\alpha}}{d\sigma}=-\frac{4}{\sigma^3}\operatorname{Cov}(\alpha,G(\alpha)),
\end{equation}
where $\operatorname{Cov}(\alpha,G(\alpha))=\mathbb{E}[\alpha G(\alpha)]-\mathbb{E}[\alpha]\mathbb{E}[G]$.

Since $b(\alpha)$ is linear with a positive slope $(m(w)\geqslant 0)$, its integral $G(\alpha)=\frac{m}{2}\alpha^2+c\alpha$is a convex function, which makes the distribution log-convex. The probability mass concentrated at the boundaries $\{0,1\}$.

We have $G(0)=\int_0^0 b(y) dy = 0$, $G(1)=\int_0^1 b(y)dy =\frac{m(w)}{2}+c(w)$. Given that $(1+\gamma)(1-w_2(t)-\gamma w_1(t))<(1-\gamma)(1+\gamma w_2(t)-w_1(t))$, we get $G(1)<0=G(0)$, which implies that the mode of the distribution at $\alpha=0$ is higher than the mode at $\alpha=1$. The distraibution is therefore heavily skewed towards $\alpha=0$.

In the region near $\alpha=0$, the term $(\alpha-\bar{\alpha}$ is negative. As $G(\alpha)$ is near its maximum value on the interval,$G(0)$, which is above the average, the term $G(\alpha)-\bar{G}$ is positive. Therefore $\operatorname{Cov}(\alpha, G(\alpha))<0$.

In the region near $\alpha=1$. The term $(\alpha-\bar{\alpha})$ is positive. As $G(\alpha)$ is near $G(1)$, which is below the average $\bar{G}$ that was pulled up by $G(0)$, the term $(G(\alpha)-\bar{G}$ is negative. Therefore $\operatorname{Cov}(\alpha, G(\alpha))<0$. 

Since the function $(\alpha-\bar{\alpha})(G(\alpha)-\bar{G})<0$  in the regions where the probability density is highest. Its expectation $\operatorname{Cov}(\alpha,G(\alpha)<0$. Therefore for all possible $w_t$, when $\sigma_1>\sigma_2$,
\begin{equation}
    \mathbb{E}_\alpha[F(w,\alpha)]_{\sigma_1}>\mathbb{E}_\alpha[F(w,\alpha)]_{\sigma_2}.
\end{equation}

By the law of total expectation, we decompose \begin{equation}
\begin{aligned}
\mathbb{E}[F(w_t,\alpha_t)]_{\sigma_1}&= \mathbb{E}_{w_t}[\mathbb{E}_{\alpha_t}[F(w_t,\alpha_t)\mid w_t]]_{\sigma_1}\\
& >\mathbb{E}_{w_t}[\mathbb{E}_{\alpha_t}[F(w_t,\alpha_t)\mid w_t]]_{\sigma_2}\\
& = \mathbb{E}[F(w_t,\alpha_t)]_{\sigma_2}.
\end{aligned}
\end{equation}
Thus, we obtain that if $\sigma_2>\sigma_1$, then $\mathbb{E}[F(w_t,\alpha_t)\mid\sigma_2]\geqslant\mathbb{E}[F(w_t,\alpha_t)\mid\sigma_1]$.

Substituting this into \textbf{Eq.~\ref{w_1infty}} completes the proof.
\end{proof}
The quantity $\mathbb{E}[w_{1}(\infty)-w_{2}(\infty)]$ characterizes the eventual gap in training intensity between the shortcut subnetwork and the core subnetwork. The theorem shows that data noise amplifies shortcut bias, whereas the noise introduced by stochastic optimization reduces it.
\section*{Lemmas Used in Proofs}
\label{app:theorem}
\begin{theorem}[Gerschgorin~\cite{HornJohnson2013}]
Let $A = [a_{ij}] \in M_n$, let
\begin{equation} \tag{6.1.1a}
R'_{i}(A) = \sum_{j \neq i} |a_{ij}|, \quad i = 1, \dots, n
\end{equation}
denote the deleted absolute row sums of $A$, and consider the $n$ Gerschgorin discs
\begin{equation}
\{z \in \mathbb{C} : |z - a_{ii}| \le R'_{i}(A)\}, \quad i = 1, \dots, n
\end{equation}
The eigenvalues of $A$ are in the union of Gerschgorin discs
\begin{equation} \tag{6.1.2}
G(A) = \bigcup_{i=1}^{n} \{z \in \mathbb{C} : |z - a_{ii}| \le R'_{i}(A)\}
\end{equation}
\end{theorem}

\begin{theorem}[Davis--Kahan $\sin\Theta$ Theorem~\cite{StewartSun1990}]
\label{thm:DK_vector}
Let \(A\in\mathbb{R}^{n\times n}\) be a real symmetric matrix with eigen-values
\begin{equation}
\lambda_{1} > \lambda_{2} \ge \cdots \ge \lambda_{n},
\end{equation}
and let \(\boldsymbol{u}_{1}\) be a unit eigen-vector associated with \(\lambda_{1}\).
Define the spectral gap
\begin{equation}
\Delta \;=\; \min_{j\ge 2}\,|\lambda_{1}-\lambda_{j}| \;>\; 0 \end{equation}
Let \(\widetilde{A}=A+E\) be another real symmetric matrix obtained by a perturbation \(E\),
and let \(\widetilde{\boldsymbol{u}}_{1}\) be a unit eigen-vector of \(\widetilde{A}\) corresponding to its largest eigen-value.
Denote by
\(
P := \boldsymbol{u}_{1}\boldsymbol{u}_{1}^{\!\top}
\)
the orthogonal projector onto \(\operatorname{span}\{\boldsymbol{u}_{1}\}\).
Then the acute angle \(\theta = \angle(\boldsymbol{u}_{1},\widetilde{\boldsymbol{u}}_{1})\) satisfies
\begin{equation}
\sin\theta
\;\;\le\;\;
\frac{\bigl\|(I-P)\,E\,\boldsymbol{u}_{1}\bigr\|_{2}}{\Delta}\; .
\label{DK-vec}
\end{equation}
\end{theorem}

\begin{lemma}
    Denote the set of all $z-$ tree by $T_z$, and we take the summation of the products of the edge weight for all $z-$trees $\in T_z$ as $q_z = \sum\limits_{h\in T_z}\prod\limits_{(j\rightarrow i)\in h}p_{ij}$. Let $\bb{q}$ be the vector $\bb{q}\equiv (q_0, q_1,\cdots,q_N)$. We have $\bb{q}$ is proportional to $\mu^*$.
\label{prop}
\end{lemma}

\begin{proof}
    Now we construct some connected graphs with $N+1$  vertices and $N+1$ edges. From graph theory, we know that such a graph contains exactly one loop. We denote the cycle containing $z$ as $z-$loop. Denote the set of all $z$-loop by $G_z$. Then we take the summation of the products of the edge weight for all $z-$loop.

\begin{equation}
q_z^\prime = \sum\limits_{g\in G_z} \prod_{(j\rightarrow i)\in g}~p_{ij}.
\end{equation}
It is easy to show that
$q_z^\prime = \sum\limits_{k\neq z} q_k p_{kz}=\sum\limits_{l\neq z} q_zp_{zl}$, then we get
\begin{equation}
\sum\limits_{k}q_kp_{kz}=\sum\limits_{k\neq z} q_kp_{kz}+q_zp_{zz}=\sum\limits_{l\neq z} q_zp_{zl}+q_zp_{zz}=q_z,
\end{equation} which shows that $q$ is proportional to $\mu^*$.
\end{proof}

\begin{lemma}[\cite{2147050d-b792-3607-85c8-fde290f942d5}]
    $c_{ij}=\vert b(i)-j\vert$
\label{weight}
\end{lemma}
The following definitions and lemmas (\textbf{Def.~\ref{dobrushin},\ref{weak ergodicity}};~\textbf{Lem.~\ref{contractofD},\ref{condofWeak}}) with a more detailed discussion can be found in Chapter 4 of the monograph~\citep{seneta2006non}.
\begin{definition}[Dobrushin coefficient]

The Dobrushin coefficient of a stochastic matrix $\bb{P}$ is defined as 

\begin{equation}
        \delta(\bb{P})\;=\,\frac{1}{2}\max\limits_{i,i^\prime} \sum\limits_{k=1}^m\vert P(i,k)-P(i^\prime,k)\vert.
\end{equation}
It can also be calculated as 
\begin{equation}
        \delta(\bb{P})\;=1-\min\limits_{i,j}\sum\limits_{s=1}^n \min(\bb{P}(i,s),\bb{P}(j,s))
\end{equation}
\label{dobrushin}
\end{definition}
\begin{definition}
    Let $\bb{T}_{0,r}$ denote the forward products of the transition matrices of the non-time-homogeneous Markov chain, defined as:
    \begin{equation}
        \bb{T}^{(p,r)}\equiv \bb{P}_p \bb{P}_{p+1}\cdots \bb{P}_{p+r-1}=\prod\limits_{k=p}^{p+r-1} \bb{P}_k.
    \end{equation}
We shall say that \emph{weak ergodicity} obtains for the MC (i.e., sequence of stochastic matrices $\bb{P}_i$) if 
    \begin{equation}
        t_{i,s}^{(p,r)}-t_{j,s}^{(p,r)}\rightarrow 0
    \end{equation}
    as $r\rightarrow \infty$ for each $i,j,s,p$, where $t_{i,s}^{(p,r)}$ is the element in the $i$-th row and $j$-th column of matrix $\bb{T}^{(p,r)}$.
\label{weak ergodicity}
\end{definition}
\begin{lemma}
    For stochastic matrix $\bb{P}=\{p_{ij}\}$
    \begin{equation}
        \sup\limits_{\substack{\Vert \bb{x}^{\prime}\Vert_{1}=1\\ \bb{x}^\prime\boldsymbol{1} = 0}}\Vert\bb{x}^\prime\bb{P}\Vert_{1}= \delta(\bb{P}),
    \end{equation} where $\delta(\cdot)$ is the Dobrushin coeffient.

\label{contractofD}
\end{lemma}
\begin{lemma}
    Then weak ergodicity of forward products $\bb{T}_{p,r}$ is equivalent to 
    \begin{equation}
    \delta(\bb{T}_{p,r})\rightarrow 0,\; r\rightarrow\infty,\;p\geqslant0.
    \end{equation}
where $\delta(\bb{T}_{p,r})$ is the Dobrushin coefficient of $\bb{T}_{p,r}$. In other words,
    \begin{equation}
        \Vert (\bb{x}-\bb{y})\bb{P}\Vert_1\leqslant
        \Vert \bb{x}-\bb{y}\Vert_1\delta(\bb{P}).
    \end{equation}
\label{condofWeak}
\end{lemma}

\begin{lemma}[Uniform limit under arbitrary switching]\label{thm:switching-limit}
Let
\begin{equation}
  \mathcal P(\varepsilon)
    \;=\;
  \bigl\{\bb{P}_{1}(\varepsilon),\ldots,\bb{P}_{K}(\varepsilon)\bigr\},
  \qquad 0<\varepsilon\le\varepsilon_{0},
\end{equation}
be a family of $m\times m$ finite row-stochastic matrices.Each matrix $\bb{P}(\varepsilon)$ is determined by the parameters of the Darwinian dynamics(\textbf{Eq.~\ref{darwanian}} in the main text).Assume that there is a probability vector $\bb{\mu}^{*}\in\Delta_{m}$ with
        \begin{equation}
          \lim_{\varepsilon\to 0}\;
          \lim_{t\to\infty}\bb{\mu}\,P_{j}(\varepsilon)^{t}
          \;=\;
          \bb{\mu}^{*},
          \quad
          \forall~j\in\{1,\ldots,K\},
          \;
          \forall~\bb{\mu} \in\Delta_{m}.
        \end{equation}
    $\bb{\mu}^*$ depends on $\alpha_0(\bb{P}_j(0))$ and $\alpha_N(\bb{P}_j(0))$, which represent the basins of attraction for state $0$ and $N$ under $\bb{P}_j(0)$. 
    Denote
    \begin{equation}
        \begin{aligned}
            \alpha_0^{\min} &= \min\{\alpha_0(\bb{P}_1(0)),\cdots,\alpha_0(\bb{P}_K(0))\}\\
            \alpha_0^{\max} &= \max\{\alpha_0(\bb{P}_1(0)),\cdots,\alpha_0(\bb{P}_K(0))\}\\
            \alpha_N^{\min} &= \min\{\alpha_N(\bb{P}_1(0)),\cdots,\alpha_N(\bb{P}_K(0))\}\\
            \alpha_N^{\max} &= \max\{\alpha_N(\bb{P}_1(0)),\cdots,\alpha_N(\bb{P}_K(0))\}\\       \end{aligned}
    \end{equation}
When $\alpha_0^{\min}>\alpha_N^{\max}$, 
     $\bb{\mu}^*=(1,0,\cdots,0).$
    Conversely, when $\alpha_N^{\min}>\alpha_0^{\max}$, 
     $\bb{\mu}^*=(0,0,\cdots,1).$
Then, for \emph{any} switching sequence
\begin{equation}
  \bb{Q}_{0}(\varepsilon),\bb{Q}_{1}(\varepsilon),\ldots ,
  \bb{Q}_{T-1}(\varepsilon)\in\mathcal P(\varepsilon)
\end{equation}
we have 
\begin{equation}\lim_{\varepsilon\to 0} \left( \limsup_{T\to\infty} \Vert \bb{\mu} \prod_{t=0}^{T-1}\bb{Q}_{t}(\varepsilon) - \bb{\mu}^{*} \Vert_{1} \right) = 0, \quad (\forall \bb{\mu}\in\Delta{m}).\end{equation}
\label{lemma-switching}
\end{lemma}
\begin{proof} Without loss of generality, assume that $\alpha_0^{\min}>\alpha_N^{\max}$.  We first analyze the case where the initial probability distribution assigns its entire mass of $\Theta(1)$ to the minimal basins of attraction of states $0$ and $N$. While for other regions, the initial probability mass is $O(\varepsilon)$. First, we show that for a non-homogeneous Markov process with such a initial probability distribution, the total probability mass on the intermediate states $\{1,2,\cdots,N-1\}$ rapidly drops to $O(\varepsilon)$.

We denote $\Omega_1^\prime = \{0,N\}$, $\Omega_2^\prime =\{1,2,\cdots,N-1\}$. $\Omega_1^\prime$ is the corresponding absorbing state when $\varepsilon=0$. Let the sum of the probability masses of $\Omega_1^\prime$ be $\bb{\mu}_{\Omega_1^\prime}=\bb{\mu }_0+\bb{\mu}_N$ and $\bb{\mu}_{\Omega_2^\prime}=\bb{\mu}_1+\bb{\mu}_2+\cdots +\bb{\mu}_{N-1}$.

Then for any $t$, we have the following equation for $\bb{\mu}_{\Omega_2^\prime}$:
\begin{equation}
    \bb{\mu}_{\Omega^{\prime}_2}(t+1) = \bb{\mu}_{\Omega^{\prime}_1}(t)\cdot R(\Omega_1^\prime\rightarrow\Omega_2^\prime)+\bb{\mu}_{\Omega^{\prime}_1}(t)\cdot R(\Omega_2^\prime\rightarrow\Omega_1^\prime).
\end{equation}
where $R(\Omega_i^\prime\rightarrow\Omega_j^\prime)$ denotes the mass transfer rate from $\Omega_i^\prime$ to $\Omega_j^\prime$. We have $R(\Omega_1^\prime\rightarrow \Omega_2^\prime)=O(\varepsilon)$ while $R(\Omega_2^\prime\rightarrow \Omega_1^\prime)=O(1)$. Thus, 
\begin{equation}
\begin{aligned}
\bb{\mu}_{\Omega_2^\prime}(t+1) &\leqslant \bb{\mu}_{\Omega_1^\prime}(t)O(\varepsilon) +\bb{\mu}_{\Omega_2^\prime}(t)(1-O(1)) + O(\varepsilon)\\
&= (1-\bb{\mu}_{\Omega_2^\prime}(t))O(\varepsilon) +\bb{\mu}_{\Omega_2^\prime}(t)(1-O(1))\\
& \leqslant  (1-O(1))\bb{\mu}_{\Omega_2^\prime}(t)) + O(\varepsilon).
\end{aligned}
\end{equation}
Therefore, $\bb{\mu}_{\Omega_2^\prime}(t+1)-O(\varepsilon)\leqslant (1-O(1))(\bb{\mu}_{\Omega_2^\prime}(t)-O(\varepsilon))$. This implies 
\begin{equation}\bb{\mu}_{\Omega_2^\prime}(t+1)\leqslant (1-O(1))^t (\bb{\mu}_{\Omega_2^\prime}(0)-O(\varepsilon))+O(\varepsilon).
\end{equation}
So in $O(\log(\frac{1}{\varepsilon}))$ time, the probability mass in $\Omega_{2}^\prime$ will drop to the order of $O(\varepsilon)$.

Next, once the mass is concentrated at states $0$ and $N$, we consider the mass transfer between them. Note that for any transition probability matrix, the rate of mass leakage from state $0$ to state $N$ is at most $O(\alpha_0^{\min})$, and from state $N$ to state $0$ is at least $O(\alpha_N^{\max})$.
Therefore, 
\begin{equation}
\begin{aligned}
\bb{\mu}_N(t+1) &= \bb{\mu}_0(t) R(0\rightarrow N)+\bb{\mu}_N(t)(1-R(N\rightarrow 0))\\
&\leqslant (1- \bb{\mu}_N(t))O(\varepsilon^{\alpha_0^{\min}}) + \bb{\mu}_N(t)(1-O(\varepsilon^{\alpha_N^{\max}}))\\
& = \bb{\mu}_N(t)(1-O(\varepsilon^{\alpha_N^{\max}})-O(\varepsilon^{\alpha_0^{\min}})) + O(\varepsilon^{\alpha_0^{\min}})\\
&\leqslant \bb{\mu}_N(t)(1-O(\varepsilon^{\alpha_N^{\max}}))+O(\varepsilon^{\alpha_0^{\min}}).
\end{aligned}
\end{equation}
This implies 
\begin{equation}
    \bb{\mu}_N(t+1)-O(\varepsilon^{\alpha_0^{\min}-\alpha_N^{\max}})\leqslant (1-O(\varepsilon^{\alpha_N^{\max}}))(\bb{\mu}_N(t)-O(\varepsilon^{\alpha_0^{\min}-\alpha_N^{\max}})),
\end{equation}
which shows that
\begin{equation}
    \bb{\mu}_N(t+1)\leqslant (1-O(\varepsilon^{\alpha_N^{\max}}))^t + O(\varepsilon^{(\alpha_0^{\min}-\alpha_N^{\max})}).
\end{equation}
So in at most $O(\varepsilon^{-\alpha_N^{\max}}\log(\frac{1}{\varepsilon}))$,the probability mass in $N$ will drop to the order of $O(\varepsilon)$. Thus, for any $\epsilon>0$, there exists a constant $t_0=O(\varepsilon^{-\alpha_N^{\max}}\log(\frac{1}{\varepsilon}))$ such that for all $T>t_0$, all vectors $\bb{\mu}$ with the initial conditions satisfy
\begin{equation}
\Vert \bb{\mu}\prod\limits_{t=0}^{T-1}\bb{Q}_t(\varepsilon)-\bb{\mu}^*\Vert_1\leqslant \varepsilon.
\end{equation}

Let $\overline{\delta}$ be the minimum Dobrushin coefficient among all matrices in the family $\mathcal{P}(\varepsilon)=\{\bb{P}_1(\varepsilon),\cdots,\bb{P}_K(\varepsilon) \}$. Since every element of $\bb{P}(\varepsilon)$ is strictly positive, it can be verified that its Dobrushin coefficient is less than 1:
\begin{equation}
\begin{aligned}
    \delta(\bb{P}(\varepsilon))\;&=\;1-\min\limits_{i,j}\sum\limits_{s=1}^n \min(\bb{P}(\varepsilon)(i,s),\bb{P}(\varepsilon)(j,s))<1.
\end{aligned}
\end{equation}
Thus, according to \textbf{Lem.~\ref{contractofD}}, we have
\begin{equation}
\begin{aligned}
\Vert(\bb{\mu}^\prime-\bb{\mu})\prod\limits_{{t=0}}^{\tilde{T}-1}\bb{Q}_t(\varepsilon)\Vert_1&\leqslant \delta(\prod\limits_{{t=0}}^{\tilde{T}-1}\bb{Q}_t(\varepsilon))\Vert\bb{\mu}^\prime-\bb{\mu}\Vert_1\\&\leqslant \prod\limits_{t=0}^{\tilde{T}-1}\delta(\bb{Q}_t(\varepsilon)\Vert\bb{\mu}^\prime-\bb{\mu}\Vert_1\\
& \leqslant (\overline{\delta})^{\tilde{T}}\Vert\bb{\mu}^\prime-\bb{\mu}\Vert_1.\\
\end{aligned}
\end{equation}
Therefore, there exists a constant $t_1=O(\frac{\log(\frac{1}{\varepsilon})}{1-\overline{\delta}})$, such that for all $\tilde{T}>t_1$, for all $\mu^{\prime}\in \Delta_m$, 
\begin{equation}
    \Vert(\bb{\mu}^\prime-\bb{\mu})\prod\limits_{{t=0}}^{\tilde{T}-1}\bb{Q}_t(\varepsilon)\Vert_1\leqslant \varepsilon.
\end{equation}
Choose $t^*=\max \{t_0,t_1\}$, for all $\bb{\mu}^\prime\in\Delta_m$,
\begin{equation}
\begin{aligned}
    \Vert \bb{\mu}^\prime\prod\limits_{t=0}^{T-1}\bb{Q}_t(\varepsilon)- \bb{\mu}^*\Vert_1&= \Vert \bb{\mu}^\prime\prod\limits_{t=0}^{T-1}\bb{Q}_t(\varepsilon)- \bb{\mu}\prod\limits_{t=0}^{T-1}\bb{Q}_t(\varepsilon)+\bb{\mu}\prod\limits_{t=0}^{T-1}\bb{Q}_t(\varepsilon)-\bb{\mu}^*\Vert_1\\
    &\leqslant \Vert \bb{\mu}^\prime\prod\limits_{t=0}^{T-1}\bb{Q}_t(\varepsilon)- \bb{\mu}\prod\limits_{t=0}^{T-1}\bb{Q}_t(\varepsilon)\Vert_1+\Vert\bb{\mu}\prod\limits_{t=0}^{T-1}\bb{Q}_t(\varepsilon)-\bb{\mu}^*\Vert_1\\
    &\leqslant \varepsilon+\varepsilon = 2\varepsilon.
\end{aligned}
\end{equation}
Therefore, $\limsup\limits_{T\rightarrow\infty}\Vert\bb{\mu}^\prime\prod\limits_{t=0}^{T-1}\bb{Q}_t(\varepsilon)- \bb{\mu}^*\Vert_1\leqslant 2\varepsilon$, and we conclude that
\begin{equation}\lim_{\varepsilon\to 0} \left( \limsup_{T\to\infty} \Vert \bb{u} \prod_{t=0}^{T-1}\bb{Q}_{t}(\varepsilon) - \bb{u}^{*} \Vert_{1} \right) = 0, \quad (\forall u\in\Delta{m}).\end{equation}
\end{proof}
\begin{lemma}
    If the evolution of $z_t\geqslant \lceil \tau B\rceil$ satisfying \textbf{Eq.~\ref{darwanian}} (in the main text), and we calculate $\pi_A, \pi_B$ according to \textbf{Eq.~\ref{sgd-dy}} (in the main text), then we can find the vector $\bb{K}$, a realization of the multivariate hypergeometric distribution $\text{MHG}(N,z_t,B,\frac{N}{B})$ for which $\pi_A(z_t)>\pi_B(z_t)$.
\label{lemma-Realbasin}
\end{lemma}
\begin{proof}
    The core idea of the proof is to make the individuals adopting the core strategy $(A)$ as concentrated as possible. Accordingly, a realization of $K$ is as follows:
    \begin{equation}
        (B,B,\cdots,B,z_t+B-\lceil \frac{z_t}{B}\rceil B,0,\cdots, 0),
    \end{equation}
    whose probability is 
    \begin{equation}
    \begin{aligned}
        &~~\quad P(k_1 = B,k_2=B,\cdots,k_{\lceil \frac{z_t}{B}\rceil-1}=B,k_{\lceil\frac{z_t}{B}\rceil} = z_t+B-\lceil\frac{z_t}{B}\rceil,0,\cdots,0)\\&=\frac{\binom{B}{B}\binom{B}{B}\cdots\binom{B}{z_t+B-\lceil\frac{z_t}{B}\rceil B}\binom{B}{0}\cdots\binom{B}{0}}{\binom{n}{z_t}}\neq 0
    \end{aligned}
    \end{equation}
In this way, we have concentrated the $z_t$ samples into $\lceil \frac{z_t}{B}\rceil$ groups, and 
\begin{equation}
\pi_A(z_t)=\frac{\lceil \frac{z_t}{B}\rceil-1}{\lceil \frac{z_t}{B}\rceil} a+\frac{1}{\lceil \frac{z_t}{B}\rceil}b.
\end{equation}
If $z_t\geqslant \lceil \tau B\rceil$, then \begin{equation}
\begin{aligned}\pi_A(z_t)&\geqslant \frac{\tau-1}{\tau}a+\frac{1} {\tau}b\\
&=\frac{(\frac{a-b}{a-d}-1)a}{\frac{a-b}{a-d}} + \frac{a-d}{a-b}b\\
&=\frac{(d-b)a}{a-d}+\frac{(a-d)b}{a-b}\\
&\geqslant \frac{(d-b)a}{a-b}+\frac{(a-d)b}{a-b}\\
&= d\geqslant \pi_B(z_t)
\end{aligned}\end{equation}
\end{proof}
\begin{lemma}
    If we calculate $\pi_A, \pi_B$ according to \textbf{Eq.~\ref{sgd-dy}} (in the main text), then for any given $L$, there exists

    \begin{equation}
\tilde{N} = \frac{L(a-b)-B\lfloor\frac{L}{B}\rfloor d}{a-d},
    \end{equation}
such that when $N \geqslant \tilde{N}$, it holds that
    \begin{equation}
    \pi_A(z_t) > \pi_B(z_t), \quad \forall z_t \in \{N - L, N - L + 1, \cdots, N - 1\}.
    \end{equation}
\end{lemma}

\begin{proof}
According to \textbf{Eq.~\ref{sgd-dy} (in the main text)},
\begin{equation}
    \pi_A(z_t)=\frac{\sum\limits_{j=1}^{\frac{N}{B}}\operatorname{I}(k_j \neq 0)[\frac{k_j}{ B}a+\frac{ B-k_j}{B}b]}{\sum\limits_{j=1}^{\frac{N}{B}}\operatorname{I}(k_j\neq0)}.
\end{equation}
Because
\begin{equation}
\begin{aligned}
    \text{Numerator~of~}\pi_A(z_t) &= \sum\limits_{j=1}^{\frac{N}{B}}\operatorname{I}(k_j \neq 0)[\frac{k_j}{ B}a+\frac{ B-k_j}{B}b]\\
    & = \sum\limits_{j=1}^{\frac{N}{B}} \frac{k_j}{B}a+\sum\limits_{j=1}^{\frac{N}{B}}\operatorname{I}(k_j\neq 0)\frac{B-k_j}{B} b\\
    &\geqslant \frac{(N-L)}{B} a+\sum_{j=1}^{\frac{N}{B}}(B-k_j)b-\sum\limits_{j=1}^{\frac{N}{B}}I(k_j=0)b\\
    &\geqslant \frac{(N-L)}{B} a+\frac{Lb}{B},
\end{aligned}
\end{equation}
and
\begin{equation}
\begin{aligned}
    \text{Dominator~of~}\pi_A(z_t) &= \sum\limits_{j=1}^{\frac{N}{B}}\operatorname{I}(k_j\neq 0)=\frac{N}{B}-\sum\limits_{j=1}^{\frac{N}{B}}\operatorname{I}(k_j=0)\leqslant\frac{N}{B}-\lfloor \frac{L}{B}\rfloor,
\end{aligned}
\end{equation}
we get
\begin{equation}
    \pi_A(z_t)\geqslant \frac{ \frac{(N-L)}{B} a+\frac{Lb}{B}}{\frac{N}{B}-\lfloor \frac{L}{B}\rfloor}.
\end{equation}
We can also show that 
\begin{equation}
\begin{aligned}
    \pi_B(z_t)&=\frac{\sum\limits_{k=1}^{\frac{N}{B}}\operatorname{I}(k_j\neq B)[\frac{k_j}{B}c+\frac{B-k_j}{B}d]}{\sum\limits_{k=1}^{\frac{N}{B}}\operatorname{I}(k_j\neq B)}\\
    &\leqslant \frac{\sum\limits_{k=1}^{\frac{N}{B}}\operatorname{I}(k_j\neq B) d}{\sum\limits_{k=1}^{\frac{N}{B}}\operatorname{I}(k_j\neq B)}=d.
\end{aligned}
\end{equation}
\end{proof}
Therefore, if $\pi_A(z_t)\geqslant \pi_B(z_t)$, we can utilize
\begin{equation}
    \frac{ \frac{(N-L)}{B} a+\frac{Lb}{B}}{\frac{N}{B}-\lfloor \frac{L}{B}\rfloor}\geqslant d.
\end{equation}
which means that
\begin{equation}
    N\geqslant \frac{L(a-b)-B\lfloor\frac{L}{B}\rfloor d}{a-d}:=\tilde{N{}}
\end{equation}
\section*{Detailed Calculation of Shortcut Bias}

In \textbf{Def.~\ref{defi_shortcut_feature}} of the paper, we introduced a pairwise comparison–based formulation of shortcut features. However, this formulation is not well-suited for empirical computation. In fact, if the study involves only two target features—namely, a shortcut feature ($V^\alpha$) and a core feature ($V^\beta$)—then there exists a simpler alternative way to compute shortcut bias. Specifically, one can select all samples where the shortcut feature and the core feature are in conflict. On this set, we evaluate the model’s classification accuracy, denoted as $\operatorname{Acc}$. Then $1 - \operatorname{Acc}$ can be regarded as another characterization of the shortcut bias. It is an empirical approximation of the following quantity. 
\begin{equation}
    \mathbb{E}_{\bb{X}}[\delta_{f(\bb{X}),V^\alpha(\bb{X})}\mid V^\alpha(\bb{X})\neq V^\beta(\bb{X})]
\end{equation}Our next theorem will show that, under certain conditions, this conventional computation method is numerically equivalent to the formulation we introduced.

\begin{theorem}
For a binary classification problem with $Y\in\{0,1\}$, consider two features, where $V^\alpha$ is the shortcut feature and $V^\beta$ is the core feature. Suppose the labels induced by $V^\alpha$ are $Y^\alpha$, and those induced by $V^\beta$ as $Y^\beta$
and 
\begin{equation}
\mathbb{E}\!\left[\delta_{f(X),\,Y(X)} \mid Y^\alpha(X)=Y^\beta(X)\right]=1.
\end{equation}
Assume the conflict set is $C=\{X \mid Y^{\alpha}(X)\neq Y^{\beta}(X)\}$. Suppose the conditions of class balance and conflict-set class balance are satisfied:
\begin{equation}
    P(Y=0)=P(Y=1)=\frac{1}{2},~P(Y=0\mid C)=P(Y=1\mid C)=\frac{1}{2},
\end{equation}
we have
\begin{equation}
\mathbb{E}\!\left[\delta_{f(X),\,Y^\alpha(X)} \mid Y^\alpha(X)\neq Y^\beta(X)\right]
= \frac{S_\beta(f)-S_\alpha(f)}{2}.
\end{equation}
\end{theorem}

\begin{proof}
We compute step by step as follows:
\begin{equation}
\begin{aligned}
&\mathbb{E}\!\left[\delta_{f(X), Y^\alpha(X)}\mid Y^\alpha(X)\neq Y^\beta(X)\right] \\
=&\mathbb{E}[\delta_{f(X),Y^\alpha(X)}\mid Y^\alpha(X)\neq Y^\beta(X), Y^\alpha(X^\prime)=Y^\beta(X^\prime)]\\
=& \tfrac{1}{2}\Big\{\mathbb{E}\!\left[\delta_{f(X),Y^\alpha(X)}\mid Y^\alpha(X)\neq Y^\alpha(X'),\ Y^\beta(X')=Y^\beta (X),\ Y^\beta(X')=Y^\alpha(X')\right]\\
&\quad +\mathbb{E}\!\left[\delta_{f(X),Y^\alpha(X)}\mid Y^\alpha(X)= Y^\alpha(X'),\ Y^\beta(X')\neq Y^\beta (X),\ Y^\beta(X')=Y^\alpha(X')\right]\Big\}\\
=& \tfrac{1}{2}\Big\{\mathbb{E}\!\left[\delta_{f(X),Y^\alpha(X)}\mid Y^\alpha(X)\neq Y^\alpha(X'),\ Y^\beta(X')=Y^\beta (X)\right]\\
&\quad +\mathbb{E}\!\left[\delta_{f(X),Y^\alpha(X)}\mid Y^\alpha(X)= Y^\alpha(X'),\ Y^\beta(X')\neq Y^\beta (X)\right]\Big\}\\
=& \tfrac{1}{2}\Big\{-\mathbb{E}\!\left[\delta_{f(X),f(X')}-\delta_{Y(X),Y(X')}\mid V^\alpha(X)\neq V^\alpha(X'),\ V^\beta(X')=V^\beta (X)\right]\\
&\quad +\mathbb{E}\!\left[\delta_{f(X),f(X')}-\delta_{Y(X),Y(X')}\mid Y^\alpha(X)= Y^\alpha(X'),\ Y^\beta(X')\neq Y^\beta (X)\right]\Big\}\\
=& \tfrac{1}{2}\big(S_\beta(f)-S_\alpha(f)\big).
\end{aligned}
\end{equation}
This completes the proof.
\end{proof}
Therefore, in the experiment described in \textbf{Fig.~\ref{Fig_1}}, shortcut bias is computed on the conflict set according to $$\sum\limits_{i} \frac{\delta_{f(\bb{X}_i), Y(\bb{X}_i)}}{\#\{V^\alpha(\bb{X}_i)\neq V^\beta(\bb{X}_i)\}},$$ which is an approximation of $\mathbb{E}\!\left[\delta_{f(X),\,V^\alpha(X)} \mid V^\alpha(X)\neq V^\beta(X)\right]$. The feature scores of shortcut and core features measure the degree of alignment between each type of feature and the label, whereas the noise feature in fact reflects the approximation error between the two kinds of shortcut biases, i.e., the model’s classification error rate in cases where the shortcut and core features are consistent.

\vskip 0.2in
\bibliography{sample.bib}

\end{document}